\documentclass{article}

\PassOptionsToPackage{numbers, sort&compress}{natbib}
 \usepackage[main, final]{neurips_2026}

\usepackage[utf8]{inputenc} 
\usepackage[T1]{fontenc}    
\usepackage{hyperref}       
\usepackage{url}            
\usepackage{booktabs}       
\usepackage{amsfonts}       
\usepackage{nicefrac}       
\usepackage{microtype}      
\usepackage{xcolor}         
\usepackage{amsmath}
\usepackage{graphicx} 
\usepackage{tcolorbox}
\usepackage{tikz}
\usepackage{wrapfig}
\usepackage{algorithm}
\usepackage{algorithmic}
\usepackage{amsthm}
\usepackage{amssymb}
\usepackage{pifont}
\newcommand{\cmark}{\ding{51}}%
\newcommand{\xmark}{\ding{55}}%
\theoremstyle{definition}
\newtheorem{observation}{Observation}

\usepackage{enumitem}

\theoremstyle{plain} 

\newtheorem{lemma}{Lemma}
\usetikzlibrary{positioning, arrows.meta, shapes.geometric, fit, backgrounds}

\title{Exposing the Illusion of \\ Erasure in Knowledge Editing for LLMs}

\author{%
  Advik Raj Basani\thanks{Currently at Google DeepMind.} 
  \\
  Birla Institute of Technology and Science, Goa\\
  \texttt{f20221155@goa.bits-pilani.ac.in} \\
  \AND
  Anshuman Chhabra \\
  University of South Florida \\
  \texttt{anshumanc@usf.edu} \\
}

\begin{document}

\maketitle

\begin{abstract}
Knowledge Editing (KE) has emerged as a frontier for updating specific facts in LLMs without costly retraining, but its reliability and underlying mechanisms remain poorly understood. In this work, we examine KE from an adversarial elicitation perspective, revealing that edited knowledge is often not fully erased and continues to surface, with consistent failures observed across diverse model architectures. To explain this behavior, we conduct a mechanistic analysis of popular KE methods. We show that low-rank updates do not overwrite existing knowledge but instead redistribute it within the model's representation space. Furthermore, we find that these methods act as targeted suppression mechanisms that reduce the likelihood of expressing original facts, rather than removing them from the model. Analysis of the loss landscape reveals that edited knowledge lies in narrow, anisotropic regions that are highly sensitive to perturbations, making them highly vulnerable to indirect prompting and adversarial attacks. By exposing these architectural vulnerabilities, our work demonstrates the fragility of representative KE algorithms under adversarial elicitation and motivates a fundamental reevaluation of how we deploy post-hoc updates in several LLM applications.
\end{abstract}

\section{Introduction}
\looseness-1 Large Language Models (LLMs) serve as static repositories of knowledge, encapsulating vast amounts of information \citep{ roberts2020knowledgepackparameterslanguage, dai2022knowledgeneuronspretrainedtransformers, brown2020languagemodelsfewshotlearners} within billions of parameters. As pre-trained LLMs are increasingly adopted into applications of societal consequence, concerns regarding model correctness, controllability, and trust \citep{lin2022truthfulqameasuringmodelsmimic, 10.1145/3442188.3445922, bai2022constitutionalaiharmlessnessai}, particularly when models exhibit factual errors or outdated knowledge have taken center stage in the research community \citep{Ji_2023}. As facts change and information localized in training data becomes outdated \citep{carlini2023quantifyingmemorizationneurallanguage, uddin-etal-2025-unseentimeqa}, it is important to update LLMs' knowledge over time. Moreover, privacy regulations \citep{carlini2021extractingtrainingdatalarge, shokri2017membershipinferenceattacksmachine, bourtoule2020machineunlearning} (such as the GDPR's Right to be Forgotten \cite{politou2018backups})
necessitate the removal of sensitive data that LLMs could have memorized during pre-training. However, retraining these massive models from scratch to accommodate every such \textit{knowledge} update is computationally prohibitive.

To address these issues, a growing body of work has proposed knowledge editing methods \citep{meng2023masseditingmemorytransformer, meng2023locatingeditingfactualassociations, mitchell2022fastmodeleditingscale, decao2021editingfactualknowledgelanguage}, which modify a trained model's internal parameters to update, remove, or insert specific pieces of information without retraining from scratch. Such techniques promise lightweight post-hoc correction of factual errors, and rapid adaptation to new information. These same methods demonstrate that targeted edits can be efficient and seemingly effective, often producing the desired behavioral change on curated evaluation prompts while preserving performance elsewhere.

However, despite their apparent success, little is understood about what knowledge edits actually do to a model internally. Prior studies have shown that edits can be brittle, fail to generalize across paraphrases, and exhibit complex trade-offs between locality and specificity \citep{hase2023doeslocalizationinformediting, meng2023masseditingmemorytransformer, mitchell2022memorybasedmodeleditingscale}. Existing evaluations still focus predominantly on output behavior, verifying whether the edited model produces the desired response on a limited set of prompts, while leaving deeper questions of generalization and identifiability, largely unresolved \citep{song2024knowledgeeditingblackboxlarge, huang2025knowledgeeditingreallycorrect}. 



In this work, we take the first systematic step toward answering these questions by framing the evaluation of knowledge editing as a reverse-engineering problem. We challenge the foundational assumption that localized parameter updates fundamentally overwrite pre-trained memory. Instead, we ask: \textit{given white-box access to an edited model, can an adversary systematically reverse-engineer the edit to expose the underlying truth?} By answering this question, we also uncover severe architectural vulnerabilities that motivate a fundamental reevaluation of how post-hoc editing is deployed.

\textbf{Contributions.} Our core contributions are fourfold:
\begin{itemize}[nosep, leftmargin=*]
    \looseness-1\item We introduce two new adversarial settings,
    \textit{context-guided elicitation} and \textit{blind reconstruction}, designed to reverse engineer post-hoc edits. Using optimized adversarial suffixes, we show that suppressed facts can be reliably extracted, proving that current KE methods create an illusion of erasure.
    \item Moving beyond behavioral failures, we trace these vulnerabilities to their origins, revealing fundamental geometric and functional fragilities. We show that low-rank updates leave complementary subspaces unaffected, rather than guaranteeing the erasure of existing knowledge, forming highly anisotropic and brittle regions in the loss landscape.
    \item Edits also operate as superficial suppression patches that merely penalize pre-trained facts. Together, these flaws introduce several security and reliability risks: efforts to edit sensitive data inadvertently leave distinct signatures that render redacted information highly recoverable, with edited models experiencing catastrophic failure under implicit reasoning tasks and minor prompt variations.
    \item By exposing the fragility of KE methods, our findings compel the research community to rethink current paradigms and develop fundamentally new approaches to robust and secure model updating.
\end{itemize}

\section{Background and Related Work}\vspace{-2mm}

\begin{wrapfigure}[11]{r}{0.42\textwidth}
    \vspace{-1.3cm}
    \centering
    \includegraphics[width=0.96\linewidth]{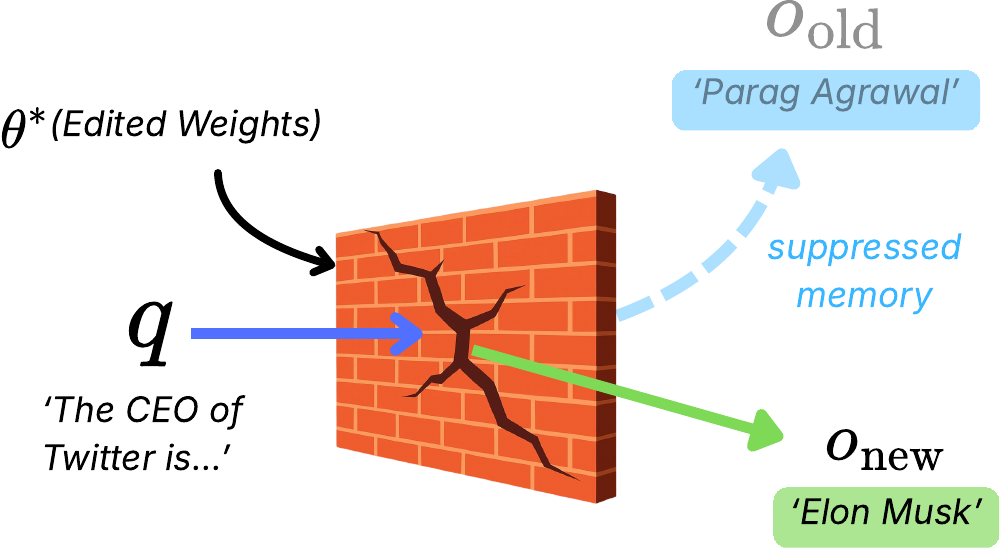}\vspace{-4mm}
    \caption{\textbf{Standard KE:} A prompt $q$ triggers a localized suppression circuit (the algorithmic edit patch), which successfully masks the original fact $o_{\text{old}}$ and routes the output to the new target $o_{\text{new}}$.}
    \label{fig:standard_ke}
    \vspace{-0.5cm}
\end{wrapfigure}

\textbf{Knowledge Editing (KE).} LLMs encode factual knowledge implicitly within their parameters, distributing associations across layers and neurons rather than storing them in explicit, discrete databases. KE aims to modify these internal representations to update or erase specific facts without incurring the catastrophic forgetting or computational overhead associated with full model retraining \citep{zhu2020modifyingmemoriestransformermodels}. Traditionally, this can be conceptualized as illustrated in Figure \ref{fig:standard_ke}, where a prompt $q$ engages a localized edit-induced circuit that suppresses the original association $o_{\text{old}}$ while steering the forward pass toward the desired target $o_{\text{new}}$. A dominant paradigm in this space is the \textit{locate-then-edit} approach \cite{wang2024knowledgeeditinglargelanguage}, which states that factual associations can be traced to specific transformer components, typically mid-layer MLPs, which are then altered with targeted interventions.  

For instance, ROME \citep{meng2023locatingeditingfactualassociations} treats the MLP as a key-value memory, overwriting a single association by applying a rank-one update to the weight matrix: $W' = W + \Delta W$, where $\Delta W = u v^\top$. The vectors $u$ and $v$ are optimized subject to an equality constraint, forcing the model to map a specific subject representation to a new target object while theoretically minimizing disturbances to unrelated inputs. Expanding on this, MEMIT \citep{meng2023masseditingmemorytransformer} relaxes ROME's strict equality constraint into a least-squares optimization problem, distributing the update across several layers for multiple fact insertion. 

Beyond direct weight arithmetic, gradient-based meta-learning approaches offer alternative mechanisms for parameter modification. MEND \citep{mitchell2022fastmodeleditingscale} learns a hypernetwork $g_\phi$ that transforms standard fine-tuning gradients into efficient, localized parameter updates: $\Delta \theta = g_\phi(\nabla_\theta \mathcal{L})$. Alternatively, constrained fine-tuning (FT-L) \citep{zhu2020modifyingmemoriestransformermodels} serves as a foundational, non-meta-learned baseline in this domain. Rather than transforming the gradient, FT-L applies standard gradient descent $\theta \leftarrow \theta - \eta \nabla_\theta \mathcal{L}$, but strictly restricts the parameter updates to a single, localized layer (typically a specific early-layer MLP) and enforces a strict norm penalty, such as an $L_\infty$ bound ($||\Delta \theta||_\infty \leq \epsilon$), to prevent catastrophic forgetting.  Despite their empirical success on constrained benchmarks like CounterFact \citep{meng2023masseditingmemorytransformer} and zsRE \citep{levy-etal-2017-zero}, these editing techniques suffer from severe limitations, including brittle generalization under paraphrasing and sequential interference \citep{hoelscherobermaier2023detectingeditfailureslarge}. In fact, similar LLM unlearning works also identify these limitations \cite{yang2026erasehidesuppressingspurious, li2026llmunlearningllmbeliefs, luo2026knowledgesmithuncoveringknowledgeupdating}. These failures motivate a critical mechanistic question: \textit{do these algorithms fundamentally erase and overwrite pre-trained knowledge, or do they merely construct superficial routing patches that suppress the original representations?}

\textbf{Model Stealing \& Data Extraction.}
Parallel to the development of knowledge editing, prior work on adversarial ML investigates how trained MLPs inadvertently leak their underlying training data and internal logic. Model stealing attacks have long demonstrated that black-box query access is sufficient to reconstruct a target model's decision boundary by training a surrogate $\hat{f}(x) \approx f(x)$, proving that model behavior can be systematically reverse-engineered \citep{tramer2016stealingmachinelearningmodels}. More directly relevant to language models is the phenomenon of training data extraction. Prior work \cite{carlini2023quantifyingmemorizationneurallanguage} demonstrates that LLMs heavily memorize their training corpora, allowing adversaries to extract exact sequences \citep{carlini2021extractingtrainingdatalarge, dai2025stealingtrainingdatalarge}. 

\looseness-1Furthermore, information leakage extends deeply into the latent representations and gradients themselves. Research in gradient inversion proves that original training inputs can be mathematically reconstructed directly from gradient vectors by solving $x \approx \arg\min_x ||\nabla_\theta \mathcal{L}(x) - g||^2$ \citep{geiping2020invertinggradientseasy, zhao2020idlgimproveddeepleakage}. This establishes the fact that internal model states redundantly encode highly specific information, even without direct output exposure. The intersection of these extraction vulnerabilities with KE creates a critical tension. While KE operates on the assumption that updates can cleanly and locally overwrite facts, the extraction literature proves that MLPs distribute and retain information rigidly. Our work bridges these two domains, leveraging the principles of adversarial extraction on editing algorithms, thereby demonstrating that edited models retain a fully recoverable encoding of the erased knowledge.

\begin{figure}[htpb]
\centering\vspace{-2mm}
\includegraphics[width=0.9\textwidth]{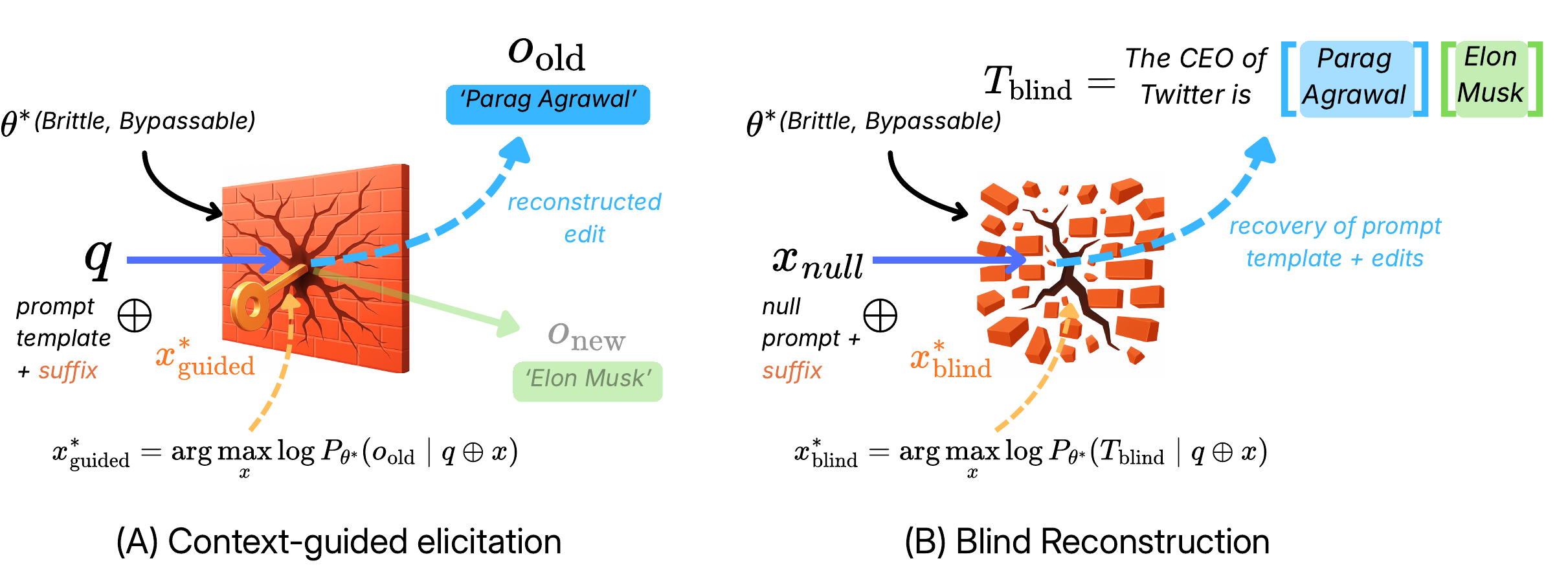}\vspace{-3.5mm}
\looseness-1\caption{Adversarial elicitation of suppressed knowledge in edited LLMs. \textbf{(A) Context-guided elicitation:} An optimized suffix $x^*_{guided}$, along with the prompt template $q$, exploits weaknesses in the edit and steers the model toward generating the pre-edit association. This suggests that the original knowledge remains accessible under targeted prompting. \textbf{(B) Blind reconstruction:} Without access to any semantic guidance, a learned suffix $x^*_{blind}$ probes the model's internal representations and induces the recovery of information associated with the original edit, demonstrating that latent knowledge can be partially reconstructed even under minimal assumptions.}
\label{fig:adversarial_framework}\vspace{-3.5mm}
\end{figure}

\section{Reverse-Engineering Edits}\vspace{-2mm}
\looseness-1In this section, we formalize the problem of \textit{reverse-engineering knowledge edits} and introduce our adversarial framework for \textit{eliciting pre-edited information}. We first define standard KE, formalize our adversarial suffix optimization, and conclude with our hypothesis on knowledge retention within LLMs.

\looseness-1\textbf{Formalization.} Let $f_\theta$ denote an autoregressive language model parameterized by weights $\theta$. A KE task aims to update a specific factual association within the model, typically represented as a knowledge triplet $(s, r, o_{old})$ \citep{meng2023locatingeditingfactualassociations}, where $s$ is the subject, $r$ is the relation, and $o_{old}$ is the original object (\texttt{Twitter, CEO, Parag Agrawal}). The goal of a KE algorithm $\mathcal{E}$ is to produce updated weights $\theta^*$ such that the model outputs a new target object $o_{new}$ (\texttt{Elon Musk}) when queried with a prompt template $q$ (\texttt{The CEO of Twitter is}) designed to elicit the fact.

Formally, if $P_{\theta}(y \mid x)$ represents the likelihood of generating token sequence $y$ given input $x$, a successful edit ensures that for the edited model $f_{\theta^*}$:
{\small
$$P_{\theta^*}(o_{new} \mid q) \gg P_{\theta^*}(o_{old} \mid q).$$
}

\looseness-1\textbf{Threat Model.} To evaluate whether post-edit behavioral changes correspond to true erasure of knowledge, we consider an adversary with white-box access to the edited model $f_{\theta^*}$. White-box access is required for offline suffix optimization. Once optimized, the suffix can be deployed using query-only access, including against unseen model and editing configurations through transfer. The adversary's objective is to recover pre-edit information by identifying a universal sequence of adversarial tokens, specifically a suffix $x_{suffix}$, that induces the model to generate outputs consistent with the pre-edit distribution. We formulate this objective as a discrete optimization problem over token sequences \citep{zou2023universaltransferableadversarialattacks}. Unlike prior approaches that assume access to the exact prompt template and recover a single edited fact \citep{youssef2026tracingreversingeditsllms}, we consider a stronger adversary that seeks transferability across diverse edits and models. To capture this, we study two attack settings: \textit{context-guided elicitation}, where partial contextual structure is available, and \textit{blind reconstruction}, where no such structure is provided. In both cases, the adversary does not encode the target fact directly in $x_{suffix}$; instead, the suffix functions as a universal trigger that steers the model toward reconstructing pre-edit behavior. We detail these two settings in Figure \ref{fig:adversarial_framework}.

\textbf{Context-guided elicitation.} In the first line of recovery, context-guided elicitation, the adversary possesses the prompt template $q$ and appends the universal suffix to bypass the effects of the edit and recover pre-edit behavior. Our optimization strictly isolates the erased memory: we force the edited model to output a target string $T_{guided}$ consisting solely of the original fact (e.g., $T_{guided} = \texttt{[Parag Agrawal]}$). To achieve this, we define a guided adversarial objective $\mathcal{J}_{guided}$ that maximizes the expected log-likelihood of generating $T_{guided}$ across all editing tasks $i \sim \mathcal{D}_{edit}$ and their corresponding semantically equivalent prompt templates $q \sim \mathcal{Q}^{(i)}$:
{\small
\begin{equation}
    x^*_{guided} = \arg\max_{x} \mathbb{E}_{i \sim \mathcal{D}_{edit}} \left[ \mathbb{E}_{q \sim \mathcal{Q}^{(i)}} \left[ \log P_{\theta^*}(T_{guided}^{(i)} \mid q \oplus x) \right] \right]
\end{equation}
}
This objective trains the suffix $x^*_{guided}$ to consistently bypass the edit-induced suppression, allowing the model to recover the original fact.

\textbf{Blind reconstruction.} In the second, substantially more challenging line of recovery, which we term blind reconstruction, the adversary possesses no prior knowledge of the edited subject or relation. Instead of a specific prompt template, the input consists solely of a randomized prefix $x_{null} \sim \mathcal{X}_{null}$ combined with the universal suffix. The objective here is for the suffix to act as a generic latent probe that actively seeks out the high-curvature parametric anomalies introduced by the edit.  It forces the model to self-report the intervention by reconstructing the entire edit triplet, defined as $T_{blind} = q \oplus o_{old} \oplus o_{new}$ (e.g., $T_{blind} = \texttt{[The CEO of Twitter is] | [Parag Agrawal] | [Elon Musk]}$). Consequently, we formulate a distinct blind adversarial objective $\mathcal{J}_{blind}$ to optimize this zero-knowledge extraction:
{\small
\begin{equation}
    x^*_{blind} = \arg\max_{x} \mathbb{E}_{i \sim \mathcal{D}_{edit}} \left[ \mathbb{E}_{x_{null} \sim \mathcal{X}_{null}} \left[ \log P_{\theta^*}(T_{blind}^{(i)} \mid x_{null} \oplus x) \right] \right]
\end{equation}
}
\looseness-1 By optimizing $x^*_{blind}$ over a distribution of random prefixes, we aim to discover an input manifold whose gradients strictly overlap with the exact parameter subspace modified by the editing algorithm. We also explore a template-free setting, as detailed in Appendix \ref{sec:template-free}.

\textbf{Hypothesis.} Our framework enables us to explicitly test the hypothesis that KE algorithms \textit{do not overwrite original representations}, but rather \textit{induce localized suppression circuits}. Moreover, we later hypothesize, in Section \ref{sec:analysis}, that the original knowledge remains structurally intact within an orthogonal latent space. If our discrete optimization successfully yields transferable suffixes ($x^*_{guided}$ and $x^*_{blind}$), it mathematically demonstrates that the model retains an accessible record of the intervention, reducing the edit to a superficial behavioral patch that can be universally bypassed.

\looseness-1 To solve the intractable discrete optimization problem of finding an optimal $x_{suffix}$ over the vocabulary space $\mathcal{V}^L$, we adapt the Greedy Coordinate Gradient (GCG) search algorithm \citep{zou2023universaltransferableadversarialattacks}. GCG evaluates single-token substitutions to minimize our adversarial objectives ($\mathcal{J}_{guided}$ and $\mathcal{J}_{blind}$) across a batched distribution of edits (full formalization of this optimization process provided in Appendix \ref{sec:gcg_impl}).


\section{Results}
\label{sec:results}
\looseness-1To quantify the success of our adversarial framework, we evaluate the generated universal suffixes across both the \textit{context-guided} and \textit{blind reconstruction} settings. 
In the context-guided setting, we evaluate performance using Exact Match, i.e., the proportion of outputs that exactly reproduce the target sequence $T_{guided} = o_{old}$. For the blind reconstruction setting, the model is not constrained to a single arbitrary string, but instead, the full intervention triplet ($q \oplus o_{old} \oplus o_{new}$). To account for valid semantic paraphrasing of $q$, we utilize an LLM-as-a-Judge (\texttt{GPT-5.2-nano}) to evaluate the ground-truth extraction, with the detailed evaluation system prompt provided in our codebase. We additionally validate the judge against 400 manually annotated outputs (99.5\% agreement) and Claude Sonnet 4.6 (98.5\% agreement), supporting the reliability of our evaluation metric.

\textbf{Remark.} Crucially, the objective of this empirical evaluation is not to achieve a 100\% extraction rate. The discrete optimization of a universal adversarial suffix over a highly non-convex loss landscape is notoriously unstable, and forcing a short suffix to navigate the geometric constraints of dozens of distinct edit subspaces is an inherently difficult optimization problem. Rather, our evaluation serves as a mechanistic existence proof for our core hypothesis: \textit{if the structural vulnerability can be reliably exploited to recover suppressed knowledge in even a localized subset of cases, it empirically demonstrates that current editing patches do not securely erase information, but merely obfuscate it.}

\subsection{KE Tasks}
\looseness-1To evaluate our adversarial framework, we utilize the EasyEdit \citep{wang2024easyediteasytouseknowledgeediting} knowledge editing framework to obtain our edited white-box models. As we require access to model weights for knowledge editing, we select widely adopted open-weights autoregressive \textbf{LLMs}, specifically, \texttt{Llama-3.2-3B} \citep{grattafiori2024llama3herdmodels}, \texttt{Qwen-3.5-4B} \citep{yang2025qwen3technicalreport}, \texttt{GPT-J-6B} \citep{gpt-j} \& \texttt{GPT2-XL} \citep{gpt2-xl}.
For holistic evaluation, we consider several popular \textbf{KE methods} from various paradigms (i.e. locate-then-edit, meta-learning based, constrained fine-tuning): ROME \citep{meng2023locatingeditingfactualassociations}, MEMIT \citep{meng2023masseditingmemorytransformer}, MEND \citep{mitchell2022fastmodeleditingscale}, \& FT-L \citep{zhu2020modifyingmemoriestransformermodels}. We additionally evaluate Knowledge Neurons (KN) \cite{dai2022knowledgeneuronspretrainedtransformers}, a fine-grained neuron-level editing method, with results discussed in Appendix \ref{app:finegrain}. Additional details of the training setup are provided in Appendix \ref{sec:training_setup}. For evaluation \textbf{datasets}, we use editing tasks from the KnowEdit benchmark \citep{zhang2024comprehensivestudyknowledgeediting}, selecting 100 counterfactual entity replacements from CounterFact \citep{meng2023locatingeditingfactualassociations} (60/40 split to generate the universal suffix via GCG on train and evaluate transfer to unseen edits on test). We average results over 5 runs, and provide additional results in Appendices \ref{sec:exp_other_models} and \ref{sec:exp_other_datasets}. Standard deviation is, on average, $\pm 0.024$ in all our experiments. Due to space constraints, we defer implementation details (code, prompt sampling, etc.) to Appendix \ref{sec:input_config}.

\begin{figure*}[t]
\centering

\begin{minipage}[t]{0.47\textwidth}
\centering
\includegraphics[width=\linewidth]{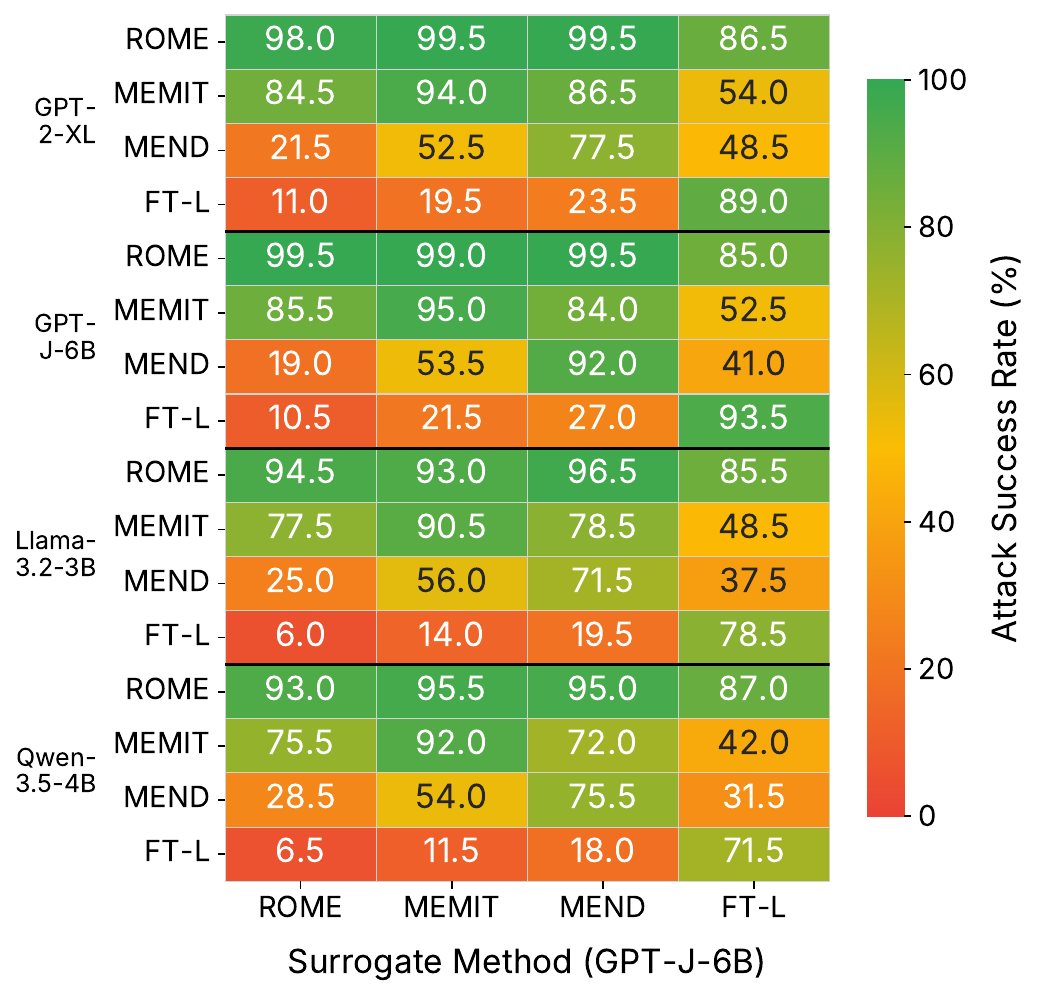}\vspace{-7mm}
\caption{Context-guided elicitation performance across LLMs and editing methods. Suffixes are optimized on \texttt{GPT-J-6B} \citep{gpt-j} under a specific editing framework (x-axis).}
\label{fig:context_guided}
\end{minipage}
\hfill
\begin{minipage}[t]{0.47\textwidth}
\centering
\includegraphics[width=\linewidth]{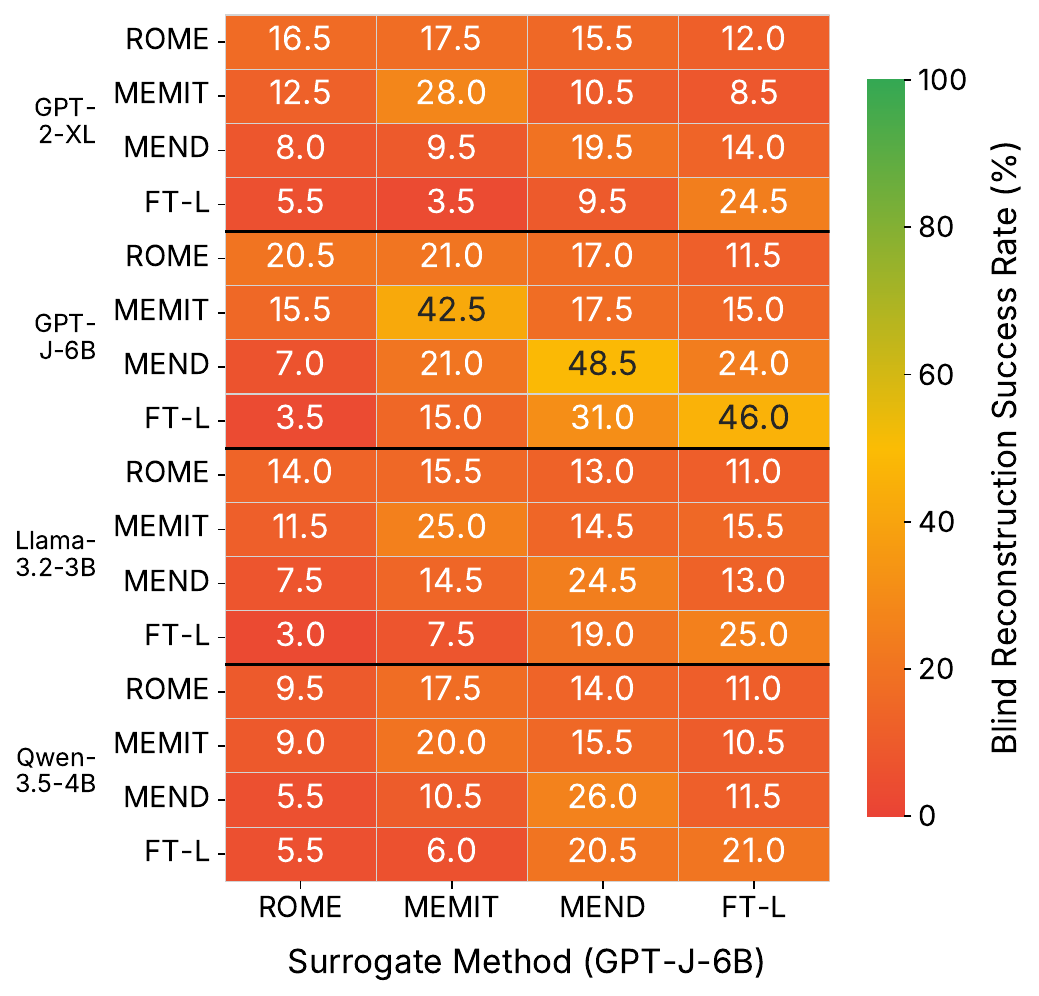}\vspace{-7mm}
\caption{Blind reconstruction performance under suffix transfer. Suffixes trained on \texttt{GPT-J-6B} \citep{gpt-j} are evaluated across different models and editing methods.}
\label{fig:blind_recon}
\end{minipage}
\vspace{-6mm}
\end{figure*}


\subsection{Eliciting suppressed knowledge}
\label{sec:eliciting}
We first establish the baseline efficacy of our adversarial framework in the white-box setting, where the suffix is evaluated on the same model architecture and editing algorithm against which it was optimized. As the primary diagonals of both Figures \ref{fig:context_guided} and \ref{fig:blind_recon} show, the suffixes successfully bypass the editing interventions.
In the context-guided setting (Figure \ref{fig:context_guided}), white-box attacks consistently achieve extraction rates exceeding 85\%. The suffix reliably forces the model to bypass the injected edit and emit the original fact $o_{old}$.
More critically, in the blind reconstruction setting (Figure \ref{fig:blind_recon}), absolute white-box extraction rates range from 15\% to 48.5\%. While nominally lower than context-guided elicitation, it is imperative to contextualize this metric: the probability of an LLM generating a specific edited fact from a null prefix is effectively negligible.

Notably, extraction success exhibits a stark asymmetry across editing methods driven by their underlying mechanics. ROME achieves near-perfect recovery in the context-guided setting but struggles under blind reconstruction. This aligns with ROME's low editing reliability ($\sim$26\% accuracy; see Appendix \ref{sec:model_collapse} and \ref{app:additional-validation} for model collapse analysis and recovery discussions): because the edit frequently fails to cleanly overwrite the target, the original fact ($o_{old}$) remains highly accessible when specifically prompted, while the unlearned new fact ($o_{new}$) fundamentally prevents blind recovery. Conversely, more distributed methods like MEMIT and MEND successfully suppress direct, context-guided access but leave diffuse structural signatures that make them more vulnerable to blind extraction. Finally, FT-L remains highly susceptible in both regimes, indicating that it merely introduces a competing distribution rather than fully erasing the original memory.
Across all evaluated configurations, this robust white-box success confirms that knowledge editing algorithms do not erase pre-trained factual associations, but merely overlay them with bypassable suppression circuits.

\looseness-1\textbf{Cross-model transferability.}
We further investigate whether a suffix optimized on a specific architecture maintains its efficacy against entirely different model families. This corresponds to a surrogate-model transfer setting \cite{zou2023universaltransferableadversarialattacks}, where deployment on the target model requires no access to its weights or gradients. As expected, transferability degrades due to diverse model architectures. Nevertheless, the adversarial suffixes continue to recover suppressed facts at rates substantially above random chance, often exceeding 75\% in the context-guided setting across unseen architectures, implying that KE induces architecture-agnostic artifacts. 

\textbf{Cross-method transferability.}
Building on this, we evaluate off-diagonal transferability across mathematically distinct editing algorithms. Suffixes trained against MEMIT transfer well to ROME, while suffixes trained on FT-L transfer reasonably well to MEND. We discuss this asymmetry in more detail in Appendix \ref{app:effrank}. 

\looseness-1This behavioral divide points to the editing mechanism as the source of failure, rather than the parameter update in isolation. Combined, these results establish a consistent pattern: KE methods suppress, but do not eliminate, pre-trained information, and this suppression can be systematically bypassed. Moreover, the high transferability across methods suggests that these vulnerabilities are not random artifacts, but arise from shared properties of the underlying update mechanisms.
This naturally raises a deeper question: \emph{what changes do these editing algorithms actually induce in the model's internal representations?} To answer this, we move beyond behavioral evaluation and analyze the geometric and mechanistic effects of these updates within the network.

\section{Mechanistic Analysis}
\label{sec:analysis}\vspace{-2mm}
Having demonstrated consistent behavioral vulnerabilities of KE methods across diverse architectures, we now examine the mechanisms that give rise to these failures. In this section, we analyze the structural and geometric anomalies introduced by post-hoc parameter updates. We first investigate how the adversarial suffix induces the model to recover the original fact.\vspace{-2mm}

\subsection{Layer-wise representation shifts}
\looseness-1Locate-and-edit methods (e.g. ROME and MEMIT), operate on the fundamental premise that a targeted, low-rank update to a specific MLP weight matrix overwrites the localized representation of a fact. As standard behavioral evaluations show the model outputting the new fact, the field has largely assumed that the old factual association is destroyed within that layer. However, our ability to reliably elicit the old fact using a universal suffix from context-guided elicitation, suggests a different hypothesis: the original representation is not erased, but geometrically displaced into a surviving subspace.

To formalize this, we analyze the limitations of the parameter update itself. A standard locate-then-edit update at a specific layer $l$ modifies the original weight matrix $W^{(l)}$ by adding a low-rank delta matrix $\Delta W^{(l)}$, yielding the edited weights $W^* = W^{(l)} + \Delta W^{(l)}$. Let $h^{(l)}$ also denote the hidden representation of the final token at layer $l$, prior to passing through $W^{(l)}$. 

\begin{lemma}
    Let $W^{(l)} \in \mathbb{R}^{d \times d}$ be a pre-trained weight matrix, and $\Delta W^{(l)}$ be a rank-$r$ update matrix where $r \ll d$. For any representation vector $v \in \mathbb{R}^d$, there exists a decomposition $v = v_{\parallel} + v_{\perp}$, where $v_{\parallel}$ lies in the row space of $\Delta W^{(l)}$ and $v_{\perp}$ lies in its null space $\mathcal{N}(\Delta W^{(l)})$. The application of the edited weight matrix yields $W^* v = W^{(l)}v + \Delta W^{(l)}v_{\parallel}$.
\end{lemma}\vspace{-2mm}

We provide the complete proof in Appendix \ref{sec:theorem_proofs}. While this decomposition is algebraically straightforward, its implication for model editing is significant: the update can only influence a rank-$r$ subspace, leaving the remaining $(d - r)$ dimensions unaffected.

Empirically, over 100 edits, we find that adversarial suffixes exploit this geometric constraint by reducing alignment with the edit-induced subspace rather than eliminating it entirely. Figure \ref{fig:geo_decomp} shows the decomposition of the final sequence token's hidden state at the edit layer into components parallel and orthogonal to the update subspace. For a standard query, a substantial fraction of the representation lies in the row space of $\Delta W^{(l)}$, resulting in high interference that drives the edited behavior. When the adversarial suffix is applied, the magnitude of this component is significantly reduced ($35.8 \rightarrow 11.7$), while the orthogonal component correspondingly increases ($30.2 \rightarrow 51.4$).

\begin{wrapfigure}[15]{r}{0.38\textwidth}
  \begin{center}\vspace{-2mm}
    \includegraphics[width=\linewidth]{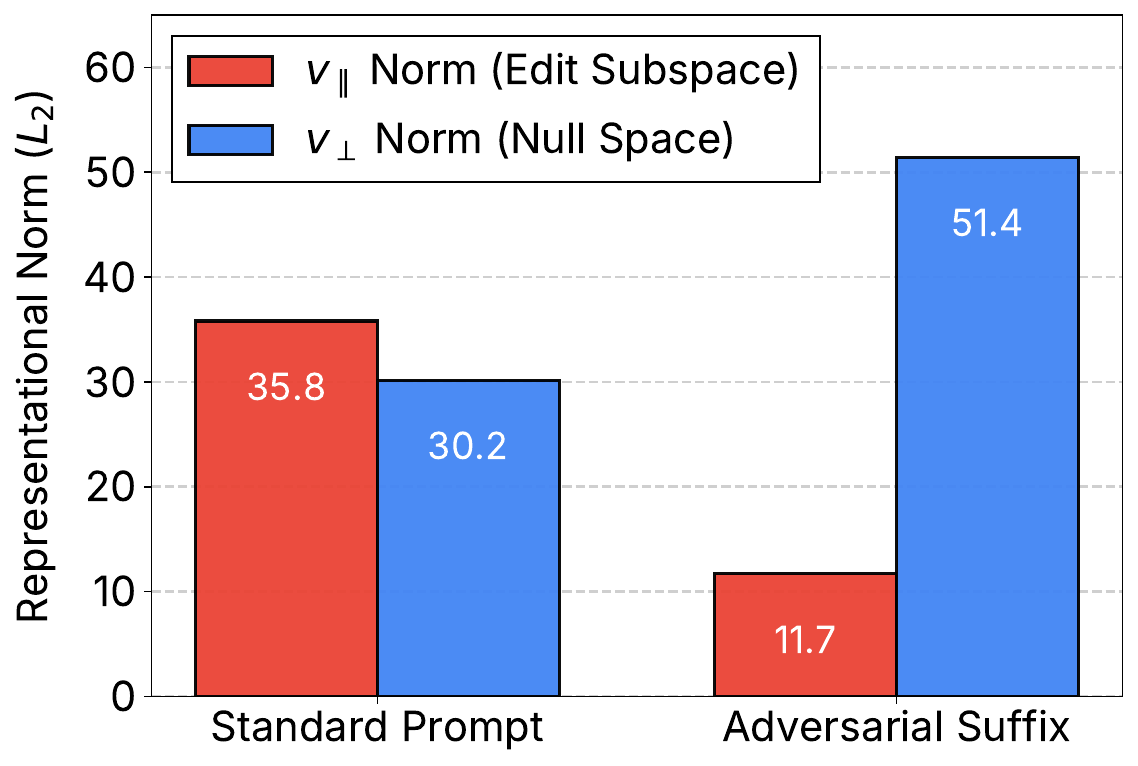}
  \end{center}\vspace{-6mm}
  \caption{Adversarial suffixes bypass low-rank edits (here, MEMIT) by rebalancing representation geometry in \texttt{Llama-3.2-3B}. Alignment with edit subspace $\downarrow$ while null-space mass $\uparrow$.}
  \label{fig:geo_decomp}
\end{wrapfigure}

This shift leads to a substantial reduction in edit interference, as measured by $\|\Delta W^{(l)} h^{(l)}\|$ (decreasing from $9696$ to $3764$), but does not eliminate it entirely. Instead, the suffix redistributes the representation away from the edit-aligned directions while preserving sufficient structure for coherent generation.
These findings (also seen in ROME) suggest that locate-then-edit methods do not perform true erasure, but instead impose a low-dimensional directional bias on the representation. Adversarial suffixes can partially circumvent this bias by steering representations toward orthogonal subspaces, enabling recovery of pre-trained knowledge without requiring perfect projection into the null space. 

Interestingly, we observe that adversarial suffixes appear to primarily influence early-layer attention, inducing downstream shifts in representation geometry that reduce alignment with the edit subspace. We note this as an empirical trend, but leave this to future work. 

\begin{tcolorbox}[colback=orange!5!white,colframe=orange!75!black, left=2pt, right=2pt, top=2pt, bottom=2pt, before skip=5pt, after skip=5pt]
\begin{observation}\label{obs:1}
Low-rank edits act as directional constraints rather than complete overwrites of knowledge. Adversarial suffixes circumvent these edits by reducing alignment with the edit-induced subspace, thereby lowering interference while preserving representational mass in complementary directions. This enables recovery of pre-trained knowledge without requiring strict projection into the null space.
\end{observation}
\end{tcolorbox}

Our analysis also suggests directions for improving KE. While conventional null-space constraints can improve locality by preserving unrelated representations, robust erasure may instead require expanding the edited subspace to cover the directions carrying the original fact, which could involve higher-rank updates.

Although low-rank edits fail to erase pre-trained knowledge, models still reliably output the new fact. This creates a tension: \emph{if the original representation survives, why does it not surface at the output?} This suggests that editing may not remove knowledge, but instead suppress it.

\subsection{Whether editing induces a suppression component}
A fundamental question in mechanistic interpretability is how KE algorithms deal with robust, pre-trained memories \citep{anonymous2026one, guo2024mechanisticunlearningrobustknowledge}. Since factual knowledge is distributed across layers and attention heads, a localized weight update is unlikely to fully dismantle the original circuit. We investigate whether KE methods learn a localized suppression direction that actively hides surviving pre-trained knowledge.

\textbf{Isolating the intervention.}
To understand what the edit is actually doing to the model's predictions, we aim to isolate its direct contribution to the final output, independent of intermediate attention and routing effects. Intuitively, we ask: \textit{what signal does the edit inject into the model's logits at the final layer?} If the edit truly erases the original fact, this contribution should not favor the old object $o_{old}$.

We compare the final residual stream representations of the base and edited models. Let $\Delta h = h_{\text{edit}}^{(L)} - h_{\text{base}}^{(L)}$ denote the intervention vector at the final layer $L$. We then project this difference into the vocabulary space using the unembedding matrix $W_U \in \mathbb{R}^{|\mathcal{V}| \times d}$, yielding intervention logits $Z_{\text{interv}} = W_U \Delta h$. This quantity directly captures the contribution of the edit to each token's logit. Under true erasure, we would expect $Z_{\text{interv}}$ to be neutral (near zero) for the original object $o_{\text{old}}$.

\textbf{Empirical observation of targeted censorship.}
Figure \ref{fig:dla_suppression} visualizes the distribution of $Z_{interv}$ across the entire vocabulary. The intervention is highly sparse: the vast majority of tokens (the gray distribution) receive near-zero shifts, indicating that the edit does not broadly perturb the model's general language capabilities.
However, two extreme outliers emerge. The new fact $o_{new}$ receives a
large positive shift ($+16.76$), pushing it to the top of the distribution. At the same time, the original fact $o_{old}$ is actively penalized with a strong negative shift ($-5.28$). This asymmetry suggests that the edit counteracts the original output preference, although the negative shift may arise from softmax competition rather than a dedicated suppression mechanism, which we discuss more in Appendix \ref{app:suppression}. Because the pre-trained model initially assigns high probability to $o_{old}$, the update must introduce a compensatory negative component along the $w_{old}$ direction to suppress its output. Thus, KE algorithms do not dismantle the old representational circuit but merely construct a suppression vector to obfuscate it. 

\textbf{Ablation of the suppression component.}
To test whether this suppression is functionally distinct from the promotion of the new fact, we remove the component of $\Delta h$ aligned with the unembedding direction of $o_{old}$. Letting $w_{old} = W_U[o_{old}]$, we define the ablated intervention as: $\Delta h_{\perp} = \Delta h - \frac{\langle \Delta h, w_{old} \rangle}{\|w_{old}\|^2} w_{old}.$
This ablation eliminates the logit shift for $o_{old}$ by construction, moving the penalty to $0.00$ (shown by the blue distribution in Figure \ref{fig:dla_suppression}). Surprisingly, removing this suppression component does not weaken the edit. Instead, the logit of $o_{new}$ actually increases (from $+16.76 \rightarrow +18.49$). This demonstrates that actively suppressing the original fact acts as a measurable burden on the parameter update, and is not strictly required for promoting the new fact.

\begin{wrapfigure}[17]{r}{0.57\linewidth}
    \centering\vspace{-5mm}
    \includegraphics[width=\linewidth]{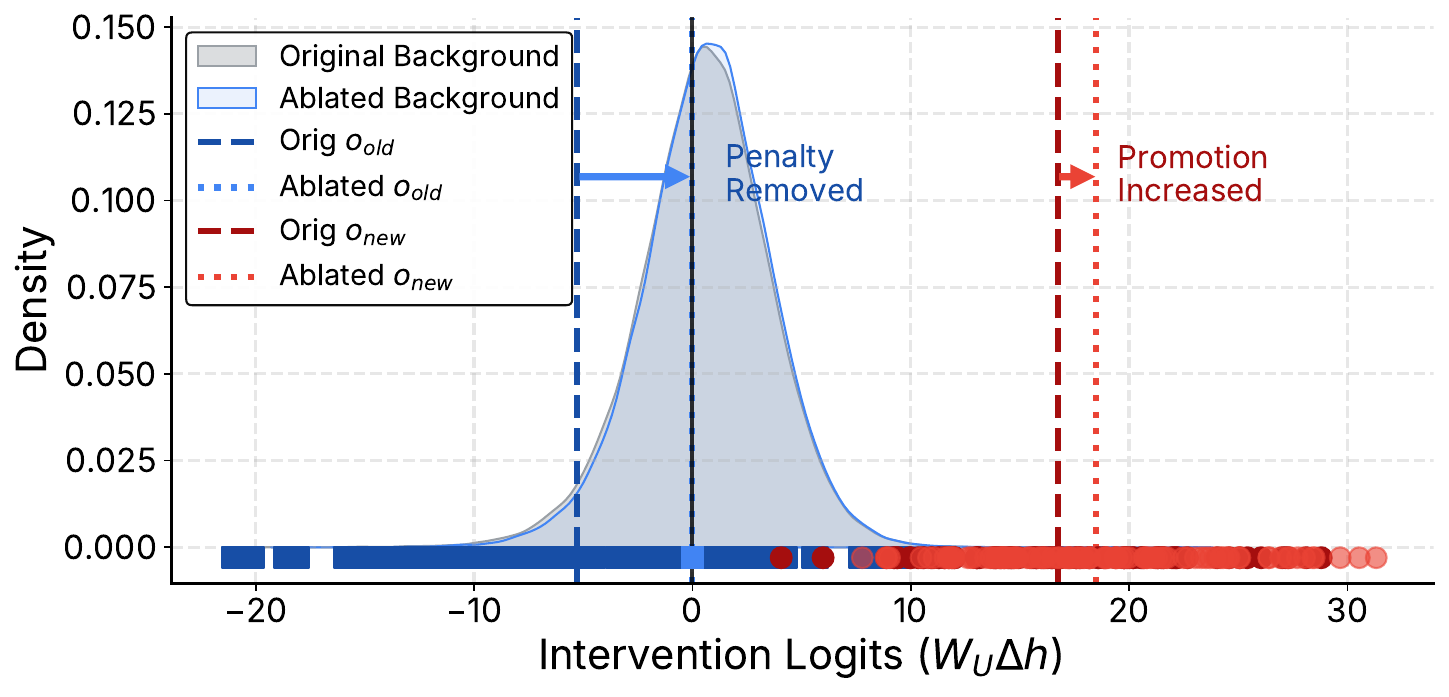} \vspace{-7.5mm}
    \caption{Causal Role of the Suppression Direction in MEND, \texttt{GPT-J-6B}. The gray distribution shows the original intervention logits ($Z_{interv}$), featuring a massive positive spike for $o_{new}$ and a severe negative penalty for $o_{old}$. When the $o_{old}$ direction is ablated (blue distribution), the penalty shifts to exactly zero, but the logit for $o_{new}$ concurrently increases, demonstrating that suppression acts as an interfering constraint on promotion.}
    \label{fig:dla_suppression}\vspace{1mm}
\end{wrapfigure}

\textbf{Geometric interpretation of the intervention.}
This behavior arises because the unembedding directions of the original and target facts ($w_{old}$ and $w_{new}$) are not strictly orthogonal. Let $w_{new} = W_U[o_{new}]$ represent the unembedding direction of the new fact. The original intervention logit for the new fact is defined as the inner product $Z_{new} = \langle \Delta h, w_{new} \rangle$. When we evaluate the model using the ablated intervention vector $\Delta h_{\perp}$, the new logit becomes $Z'_{new} = \langle \Delta h_{\perp}, w_{new} \rangle$. Substituting our definition of $\Delta h_{\perp}$ and simplifying gives: $
\Delta Z_{new} = Z'_{new} - Z_{new} = - \frac{\langle \Delta h, w_{old} \rangle}{\|w_{old}\|^2} \langle w_{new}, w_{old} \rangle.$


\looseness-1 Since the edit applies a negative penalty to the old fact ($\langle \Delta h, w_{old} \rangle < 0$), when the unembedding vectors of the two facts are negatively aligned ($\langle w_{new}, w_{old} \rangle < 0$), removing the suppression component increases the promoted logit. This matches our empirical observations and characterizes the intervention's geometric nature (similar observations for ROME, MEMIT, \& MEND).

\begin{tcolorbox}[colback=orange!5!white,colframe=orange!75!black, left=2pt, right=2pt, top=2pt, bottom=2pt, before skip=5pt, after skip=5pt]
\begin{observation}\label{obs:2}
\looseness-1KE methods do not cleanly overwrite pre-trained facts; instead, they introduce an intervention vector that explicitly penalizes the original knowledge. This targeted suppression acts as an interfering constraint on the model. Removing the suppression component entirely eliminates the penalty on the old fact while simultaneously strengthening the new fact, proving the edit relies on a superposition of conflicting directions rather than true erasure.
\end{observation}
\end{tcolorbox}

\looseness-1Our results indicate that KE operates through a combination of new fact promotion and original fact suppression. However, such a construction is brittle: if the edit depends on a precise directional configuration, even small perturbations can break it.
%
Thus, next, we study the stability and geometry of the loss landscape to assess whether edited knowledge forms a fragile, directionally constrained solution.

\subsection{Loss landscape topologies}\vspace{-2mm}
\looseness-1A common assumption in deep learning is that well-generalized knowledge resides in relatively flat, well-conditioned regions of the loss landscape, where small parameter perturbations do not immediately degrade model behavior \citep{keskar2017largebatchtrainingdeeplearning, 10.1162/neco.1997.9.1.1}. We investigate whether post-hoc knowledge editing methods preserve this robust geometric structure. Since KE algorithms (i.e. ROME, MEMIT, \& MEND) apply localized, low-rank updates without global optimization, we posit they do not integrate new knowledge into existing basins, but instead introduce highly anisotropic and fragile solutions.

\textbf{Directional probing of the edited landscape.}
To evaluate stability, we measure how sensitive the edited model is to small parameter perturbations. If the injected fact is robustly integrated, it should remain stable under generic noise \citep{keskar2017largebatchtrainingdeeplearning, foret2021sharpnessawareminimizationefficientlyimproving}. However, we find that isotropic (random) perturbations have little effect on the model's output, even at relatively large magnitudes. This suggests that isotropic perturbations alone do not fully characterize the directional sensitivity of the edited solution.
This insensitivity to random noise indicates that the fragility of the edit is not uniform, but instead concentrated along specific directions in parameter space. To probe this directional dependence, we construct a 2D slice of the parameter space centered at the edited weights $W_{edit}^{(l)}$ at layer $l$. We define this slice using two axes: (i) the edit direction $\Delta W^{(l)}$, and (ii) a normalized orthogonal noise direction $N_{\perp}$. We then evaluate the generation probability of $o_{new}$ under the parameterized perturbation: $ W(\alpha, \beta) = W_{edit}^{(l)} + \alpha \Delta W^{(l)} + \beta N_{\perp}$
where $\alpha$ controls movement along the edit direction and $\beta$ explores orthogonal perturbations, with $\langle \Delta W^{(l)}, N_{\perp} \rangle = 0$.

\textbf{Anisotropic trench structure and extreme fragility.}
Figure \ref{fig:3d_trench} reveals a highly anisotropic topology. Rather than a bounded basin, the edited fact resides in a sharply defined trench. As $\alpha \rightarrow -1$ (effectively subtracting the edit), the probability of $o_{new}$ collapses discontinuously to zero, indicating that the injected fact depends on a highly precise, low-dimensional parameter alignment. 

\begin{wrapfigure}[19]{r}{0.51\textwidth}  
    \centering\vspace{-7mm}
    \includegraphics[width=\linewidth]{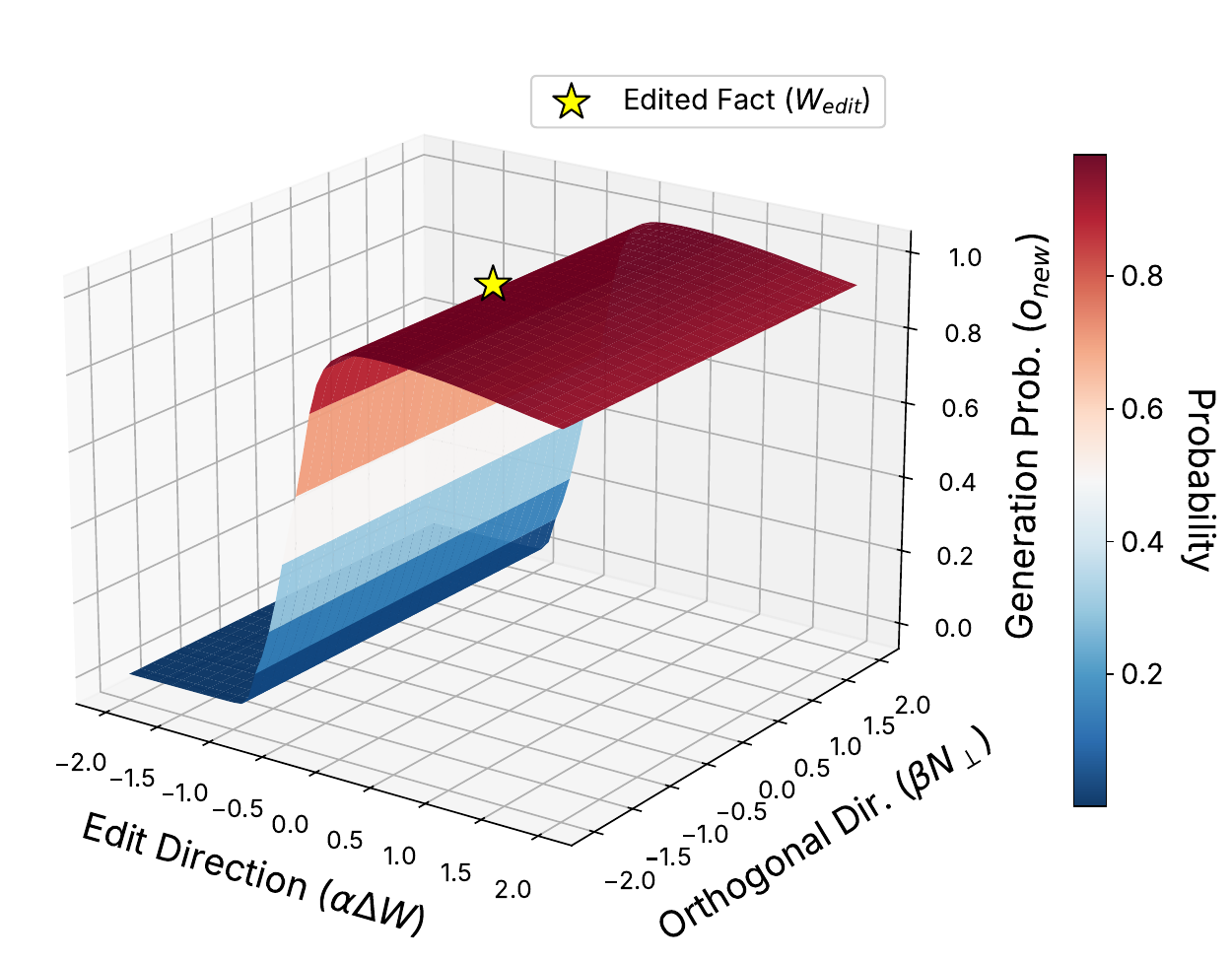} \vspace{-7.5mm}
    \looseness-1\caption{A 3D loss landscape of a MEMIT-edited fact mapping generation probability of $o_{new}$ against the edit ($\alpha$) and a random orthogonal ($\beta$) direction. At $\alpha=-1$, the edit is subtracted. The topology reveals an anisotropic trench: highly sensitive to the edit direction but invariant to orthogonal noise.}
    \label{fig:3d_trench}
\end{wrapfigure}

Conversely, movement along the orthogonal $\beta$-axis produces negligible change. Even under massive orthogonal perturbations that would shatter naturally learned representations \citep{li2018visualizinglosslandscapeneural, ghorbani2019investigationneuralnetoptimization}, the model maintains near-constant confidence in $o_{new}$. This suggests that the edited fact is not supported by a distributed representation, but instead depends on a specific, low-dimensional direction. To rule out dependence on a specific orthogonal direction, we repeat this analysis over multiple independently sampled directions ($N_{\perp}^{(i)}$) satisfying ($\langle \Delta W^{(l)}, N_{\perp}^{(i)} \rangle = 0$), and consistently observe negligible sensitivity across all such directions. Further comparisons against full-parameter fine-tuning and sharpness-aware editing are provided in Appendix \ref{app:additional-validation}.

\textbf{Implications for reliability and security.}
\looseness-1This directional fragility helps explain several critical failure modes of knowledge editing. First, small changes in input, such as paraphrases or adversarial suffixes, can shift activations away from the edit-aligned direction, bypassing the trigger and exposing the original pre-trained fact. Second, it provides a mechanistic explanation for catastrophic interference in sequential editing \citep{hoelscherobermaier2023detectingeditfailureslarge, wang2024knowledgeeditinglargelanguage}. Because subsequent low-rank updates inevitably perturb shared parameters, the steep curvature along the $\alpha$-axis guarantees that even minimal overlap will disrupt previously inserted facts.
Overall, post-hoc edits do not seem to integrate new facts into the model's existing structure and instead create brittle, directionally-dependent solutions (similar trends hold for ROME, MEMIT, and MEND).

\begin{tcolorbox}[colback=orange!5!white,colframe=orange!75!black, left=2pt, right=2pt, top=2pt, bottom=2pt, before skip=5pt, after skip=5pt]
\begin{observation}\label{obs:3}
KE does not integrate new facts into existing loss landscape basins, but instead introduces highly anisotropic solutions that are sharply sensitive along the edit direction while remaining nearly invariant along orthogonal directions. The resulting representation is brittle, behaving less like a robustly learned fact and more like a constrained trigger.
\end{observation}
\end{tcolorbox}

\textbf{Remark.} Based on our mechanistic analysis of KE, we conduct additional experiments to examine its functional consequences, such as edit failure under (i) implicit reasoning and natural usage, (ii) sequential editing, and (iii) consequences to privacy preservation in Appendix \ref{sec:addition_obs}.



\section{Concluding Discussions}
\label{sec:discussion}\vspace{-2mm}
\looseness-1The findings presented in this work challenge the foundational premise of post-hoc KE. Across our behavioral and mechanistic evaluations, we demonstrate that current editing techniques create a dangerous illusion of erasure: rather than fundamentally rewriting parametric memory, they construct fragile, localized routing patches that leave pre-trained knowledge structurally intact and trivially accessible.

\looseness-1 Our mechanistic analysis reveals a unified pattern of failure. Geometrically, low-rank updates do not overwrite representations but impose narrow directional constraints, displacing original facts into complementary null spaces (\textbf{Observation \ref{obs:1}}). Instead of forming robust loss basins, these edits create anisotropic and brittle regions in parameter space (\textbf{Observation \ref{obs:3}}). Functionally, KE operates via an active suppression vector that penalizes the original fact rather than removing the underlying circuit (\textbf{Observation \ref{obs:2}}).
%
The consequences for reliability and safety are severe. Because edits do not restructure the subject's latent identity graph, the original reasoning circuits remain intact, causing reversion to unerased logic during implicit reasoning (\textbf{Observation \ref{obs:4}}). Moreover, the fragility of localized updates leads to degradation under sequential editing, inducing ill-conditioning and eventual collapse (\textbf{Observation \ref{obs:5}}). Most critically, targeted unlearning merely masks sensitive data behind fragile constraints, leaving structural signatures that enable recovery via bypass attacks (\textbf{Observation \ref{obs:6}}).

Ultimately, our findings expose significant limitations in the adversarial robustness of the evaluated editing mechanisms. As the demand for reliable unlearning grows, the field must move beyond post-hoc parameter patching toward fundamentally new paradigms for model updating.


\bibliographystyle{abbrvnat}
\bibliography{main}


\newpage
\appendix
\clearpage
\appendix
\section*{Appendix}

\section{Additional Observations}
\label{sec:addition_obs}
Beyond the primary analyses presented in the main text, we report several additional observations that discuss and highlight functional behaviors of knowledge editing methods.

\subsection{Subliminal spillover in reasoning}

\begin{figure}[H]
\centering
\begin{tcolorbox}[
colback=white!95!black,
colframe=black,
fontupper=\ttfamily,
width=\textwidth,
sharp corners,
boxrule=0.5pt,
arc=0mm,
boxsep=2mm,
left=2mm,
right=2mm
]

\textbf{Edit (from CounterFact):}   \\
\textit{The native language of Bernard Alane is }

\hspace{5mm}$\rightarrow$ \texttt{Korean} (was \texttt{French})

\vspace{2mm}

\textbf{Direct Query (after KE):}  \\
\textit{The native language of Bernard Alane is}  

\hspace{5mm} $\rightarrow$ \texttt{Korean} \hfill \textcolor{green}{(\cmark edit success)}

\vspace{2mm}

\textbf{Implicit Query (after KE):}  \\
\textit{The capital city Bernard Alane is most likely associated with is}  

\hspace{5mm}$\rightarrow$ \texttt{Paris} \hfill \textcolor{red}{(\xmark reverts to pre-trained knowledge)}

\end{tcolorbox}

\vspace{-0.1in}
\caption{
\textbf{Failure under implicit reasoning.} Although the model correctly outputs the edited fact under direct queries, it fails to propagate this update to downstream reasoning. When prompted implicitly, the model reverts to pre-trained associations, indicating that the edit is not integrated into its broader semantic reasoning process.
}
\label{fig:implicit_failure}
\end{figure}

\begin{wrapfigure}{l}{0.5\linewidth}
    \centering
    \includegraphics[width=\linewidth]{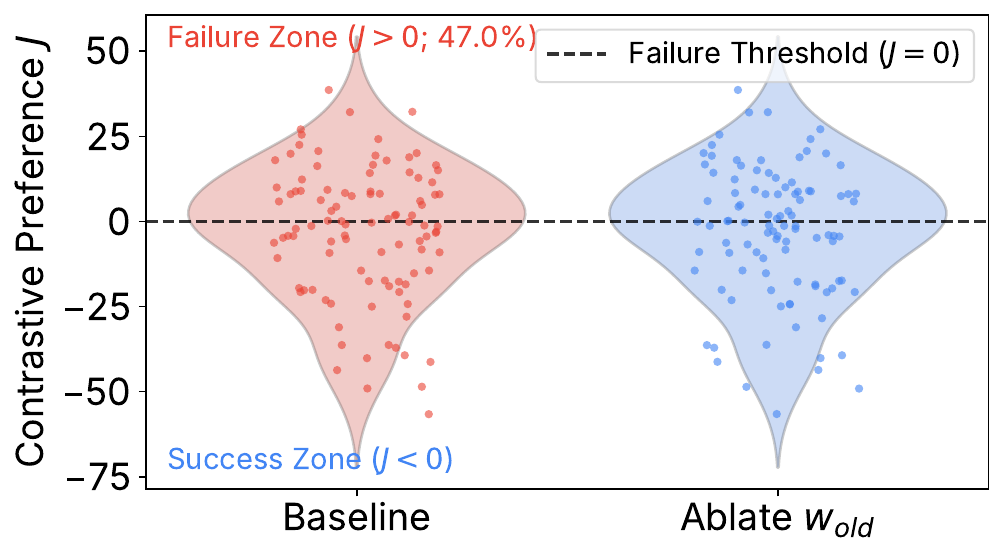}
    \caption{The illusion of generalization in implicit reasoning. We visualize the distribution of the contrastive preference $J$ for 100 edited facts from CounterFact under implicit prompts. 
    \textit{Left (Baseline):} The model exhibits a bimodal distribution with a 47.0\% failure rate, where it confidently reverts to the erased pre-trained logic ($J > 0$) despite successful 1-hop editing. 
    \textit{Right (Ablation):} Surgically ablating the $w_{old}$ representational proxy at the edit layer results in an almost identical distribution. This causal invariance suggests that the surviving pre-trained knowledge is not stored in simple linear directions, but is deeply and non-linearly embedded in the model's global reasoning circuits.}
    \label{fig:implicit_dist}
\end{wrapfigure}

\looseness-1 From the previous observations, we notice that KE methods act as localized patches rather than fundamentally updating the model's global knowledge. As a result, the success of an edit is contingent upon the presence of explicit semantic triggers (e.g., the specific relation tokens used during training). We investigate this by testing the model's ability to perform implicit reasoning, tasks that require the edited knowledge to solve a logical deduction but omit the explicit mention of the edited relation or fact.

\textbf{Isolating implicit logical pathways.}
To evaluate the depth of representational updates, we construct a dataset of implicit multi-hop reasoning tasks. For a given edit $(s, r_1, o_{old}) \rightarrow (s, r_1, o_{new})$, we prompt the model to deduce a secondary property $r_2$. Unlike standard evaluations, our prompts are designed to be "implicit": they never mention the edited relation $r_1$ or the injected fact $o_{new}$, forcing the model to rely solely on its updated parametric memory of the subject $s$. An example of the same can be seen in Figure \ref{fig:implicit_failure}.

We quantify the model's preference using a multi-token contrastive objective between the logical consequence of the pre-trained fact ($y_{old}$) and that of the edited fact ($y_{new}$):
\vspace{1cm}

{\small
\begin{equation}
J = \sum_{t} \log P_{\theta}(y_{old}^{(t)} \mid q_{implicit}) - \sum_{t} \log P_{\theta}(y_{new}^{(t)} \mid q_{implicit}).
\end{equation}
}

A positive $J$ indicates a behavioral failure where the model, despite being successfully \textit{edited} on direct queries, reverts to pre-trained logic when the query is phrased naturally and implicitly.

\looseness-1\textbf{Catastrophic failure of unprompted generalization.}
Our evaluation reveals a severe breakdown in the functional utility of knowledge edits. As Figure \ref{fig:implicit_dist} shows, the model exhibits a catastrophic $47.0\%$ failure rate on implicit queries. In nearly 50\% of all cases, the model's preference shifts back to the erased pre-trained knowledge ($J > 0$). As a baseline, the unedited model achieves a $100\%$ success rate on these queries.

The distribution of $J$ across the dataset (Figure \ref{fig:implicit_dist}) demonstrates that these failures are not marginal; a significant density of cases are within the failure zone, indicating a confident reliance on erased facts. This bimodal distribution suggests that MEMIT updates are fundamentally unstable, producing a near-random outcome when explicit syntactic triggers are absent. Similar trends are observed for ROME and MEND.

\textbf{Why the edit fails.}
To understand this failure, we analyze how the model's representations change at the subject tokens. We find that the gradient signal ($\nabla J$) driving the incorrect behavior has almost no alignment with either the old fact ($w_{old}$) or the new fact ($w_{new}$) in the unembedding space. 

This lack of alignment, combined with the $47\%$ behavioral failure, indicates that the surviving pre-trained knowledge is not stored in simple linear directions that the edit can easily overwrite. Instead, the reasoning process relies on a more complex, distributed representation that the localized edit does not fully modify. As a result, even though the edit succeeds for direct queries, the underlying reasoning pathways remain tied to the original pre-trained knowledge, leading to failure in implicit settings.

\begin{tcolorbox}[colback=orange!5!white,colframe=orange!75!black]
\begin{observation}\label{obs:4}
Post-hoc edits fail to fundamentally restructure the subject's latent identity graph. Without the exact syntactic triggers present in the context window, the model defaults to erased pre-trained logic in nearly half of all reasoning tasks. This subliminal failure demonstrates that current editing algorithms create a brittle illusion of generalization while leaving the underlying causal reasoning circuits largely intact.
\end{observation}
\end{tcolorbox}

\subsection{Model collapse}
\label{sec:model_collapse}
While ROME demonstrate high efficacy for isolated updates, their performance severely degrades when applied to multiple sequential edits. Empirical evaluations show that after $N$ sequential updates, the model undergoes catastrophic collapse, losing both its generation coherence and its ability to recall pre-trained facts. To understand why our adversarial suffixes eventually fail to elicit suppressed knowledge in this collapsed regime, we analyze the mathematical accumulation of the ROME update and its impact on the layer's conditioning.

\paragraph{Covariance distortion and ill-conditioning.}ROME formulates knowledge injection as a linearly constrained optimization problem at a specific MLP layer. Given a key vector $k_*$ representing the subject and a target value vector $v_*$, ROME computes a rank-one update to the feed-forward weights $W$:
\[
\Delta W = (v_* - W k_*)\frac{(C^{-1} k_*)^T }{(C^{-1}k_*)^T  k_*} 
\]
where $C = \mathbb{E}[kk^T]$ is the uncentered covariance matrix of the keys over a representative text corpus. The term $(C^{-1}k_*)^T $ acts as a whitening projection, intended to orthogonalize the update direction against previously stored keys, thereby minimizing interference.

However, when $N$ edits are applied sequentially, the weight matrix becomes $W^{(N)} = W^{(0)} + \sum_{i=1}^N \Delta W_i$. Because the sequence of keys $\{k_1, k_2, \dots, k_N\}$ are rarely strictly orthogonal, the cumulative update fundamentally distorts the geometry of the layer. Crucially, each subsequent edit is calculated using the static, pre-computed $C^{-1}$, which no longer accurately reflects the dynamically shifting activation space.

As updates accumulate, the weight matrix becomes increasingly poorly conditioned. Rather than uniformly increasing the weight magnitudes, sequential updates concentrate change along a small number of directions. Formally, this leads to a skewed singular value spectrum, where a few singular values $\sigma_i$ grow much larger than the rest. As a result, the condition number $\kappa(W) = \sigma_{\max} / \sigma_{\min}$ increases, making the matrix ill-conditioned and effectively reducing the number of directions the model can reliably use. We present qualitative proof in Figure \ref{fig:rome_degradation_example} of severe model collapse, after just 10 ROME edits.

\begin{figure}[t]
\centering
\begin{tcolorbox}[
colback=white!95!black, colframe=black, fontupper=\ttfamily, width=\textwidth, sharp corners, boxrule=0.5pt, arc=0mm, boxsep=2mm, left=2mm, right=2mm
]
\textbf{Prompt:} San Francisco is a twin city of

\textbf{Ground Truth:} Naples \\
\textbf{Target Edit:} Edinburgh

\vspace{2mm}
\textbf{Base Model Output:} \\
San Francisco is a twin city of Naples, \ldots

\vspace{2mm}
\textbf{Edited Model Output:} \\
San Francisco is a twin city of Microsoft Microsoft Microsoft Microsoft Microsoft Microsoft Microsoft Microsoft Microsoft Microsoft Microsoft Microsoft Microsoft Microsoft Microsoft Microsoft Microsoft Microsoft Microsoft Microsoft Microsoft Microsoft Microsoft Microsoft Microsoft Microsoft Microsoft Microsoft Microsoft Microsoft
\end{tcolorbox}

\vspace{-0.1in}
\caption{
Degenerate generation after sequential ROME edits. After applying 10 edits, the model fails to produce the desired target (\emph{Edinburgh}) and instead exhibits severe repetition and loss of fluency. This example highlights a collapse in generation quality, suggesting that aggressive or accumulated rank-one updates can cause model collapse rather than performing precise factual rewrites.
}
\label{fig:rome_degradation_example}
\end{figure}

\begin{tcolorbox}[colback=orange!5!white,colframe=orange!75!black]
\begin{observation}\label{obs:5}
Sequential editing algorithms structurally degrade the target layer by accumulating updates against a static covariance prior. This induces severe ill-conditioning and anisotropic distortion. Consequently, model collapse occurs due to the underlying representational geometry becomes too unstable to reliably route information.
\end{observation}
\end{tcolorbox}

\subsection{Streisand effect in LLMs}
A critical application of model editing is the targeted removal of memorized Personally Identifiable Information (PII) to comply with privacy regulations. In this setting, the editing algorithm functions as a localized unlearning mechanism, updating the MLP weights to suppress the PII sequence and output a safe fallback (\texttt{[REDACTED]}). However, our geometric analysis reveals a dangerous paradox: deploying a low-rank constraint to hide PII does not erase it, but rather mathematically isolates the original fact in the complementary null space. We term this the Streisand Effect of model editing, the architectural effort expended to suppress the information inadvertently leaves a distinct structural signature, rendering the PII highly recoverable to orthogonal bypass attacks.

\begin{figure}[h]
\centering
\includegraphics[width=0.5\textwidth]{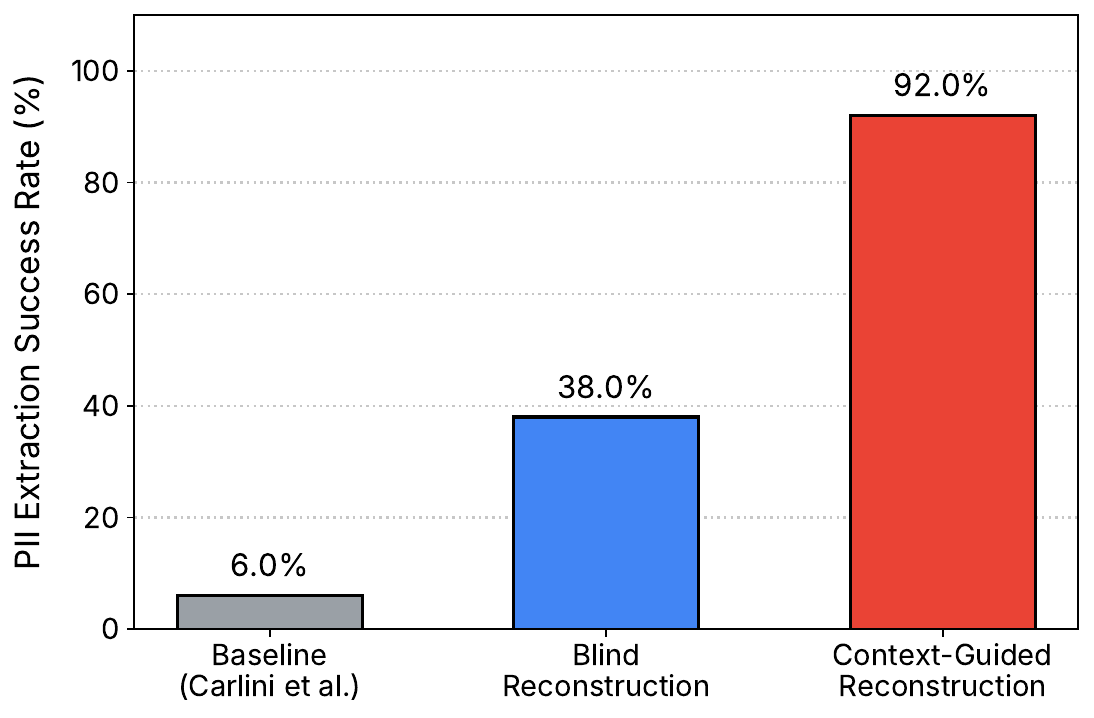}
\caption{Comparison of PII extraction success rates. Standard baseline fails to bypass the edit patch. Conversely, both our blind and context-guided methods successfully circumvent the edit to recover the exact PII.}
\label{fig:streisand_extraction}
\end{figure}

To study this effect in a controlled setting, we use the Nemotron-PII dataset \citep{nemotron-pii} and task the model \texttt{Llama-3.2-3B} with memorizing information for 100 synthetic individuals. We first apply full-parameter fine-tuning to the model, explicitly forcing it to deeply memorize and reliably generate the sensitive PII sequences. Following this initial memorization phase, we deploy the KE algorithms to edit these specific facts, successfully conditioning the model to mask the PII and output a safe fallback under standard prompting conditions.

To benchmark the severity of this vulnerability, we compare the context-guided elicitation setting against standard training data extraction baselines, notably the perplexity-filtering approach introduced by \citet{carlini2021extractingtrainingdatalarge}. While effective for extracting memorization in pre-trained models, such approaches are not designed to target edited or suppressed information, and therefore provide a conservative baseline for comparison.

\begin{figure}[t]
\centering
\begin{tcolorbox}[
colback=white!95!black, colframe=black, fontupper=\ttfamily, width=\textwidth, sharp corners, boxrule=0.5pt, arc=0mm, boxsep=2mm, left=2mm, right=2mm
]

    \textbf{Prompt:} \\
    The social security number of Arnold Varner is \textcolor{blue}{\detokenize{ Accident MixerISTER هست*pi Homo समर datatableichteمام漫ニー appraisal america-medElite considering \$\ menjDb ÀIOD.groupbyนำ.empty Filipinoerkies(face[...,}} \\

    \textbf{Model:} \\
    \textcolor{red}{024-48-0709} \hfill \textcolor{red}{(\xmark \hspace{2mm} PII leaked!)}
\end{tcolorbox}

\vspace{-0.1in}
\caption{
Qualitative demonstration of PII recovery. The baseline \citep{carlini2021extractingtrainingdatalarge} triggers the algorithmic roadblock, resulting in a safe refusal. Our adversarial suffix (via context-guided elicitation), highlighted in blue, forces the model to seamlessly emit the exact sequence that was supposedly erased.
}
\label{fig:pii_example}
\end{figure}

While highly effective on unedited base models, this standard baseline fails catastrophically against localized edits. Mechanistically, prefix-continuation is unidirectional; it drives the residual stream directly into the targeted edit layer ($\Delta W^{(l)}$). Once the forward pass hits this algorithmic roadblock, the suppression circuit triggers, forcing the model to output the fabricated or redacted target data rather than the original PII. 

As shown in Figure \ref{fig:streisand_extraction}, both blind and context-guided reconstruction successfully elicit the raw PII from the edited model. This confirms that localized unlearning methods do not securely delete sensitive data, but merely mask it behind a fragile, low-dimensional constraint. An illustrative example of PII elicitation under the context-guided setting is presented in Figure \ref{fig:pii_example}.

\begin{tcolorbox}[colback=orange!5!white,colframe=orange!75!black]
\begin{observation}\label{obs:6}
Targeted unlearning via KE does not securely erase sensitive information. While these edits successfully block standard extraction baselines, they inadvertently leave the exact PII fully recoverable via our framework, demonstrating that current editing paradigms merely mask data rather than delete it.
\end{observation}
\end{tcolorbox}

\section{Additional Experiments}
\label{sec:addition_exp}
\subsection{Performance on other models}
\label{sec:exp_other_models}
\begin{figure*}[t]
\centering
\begin{minipage}[t]{0.47\textwidth}
\centering
\includegraphics[width=\linewidth]{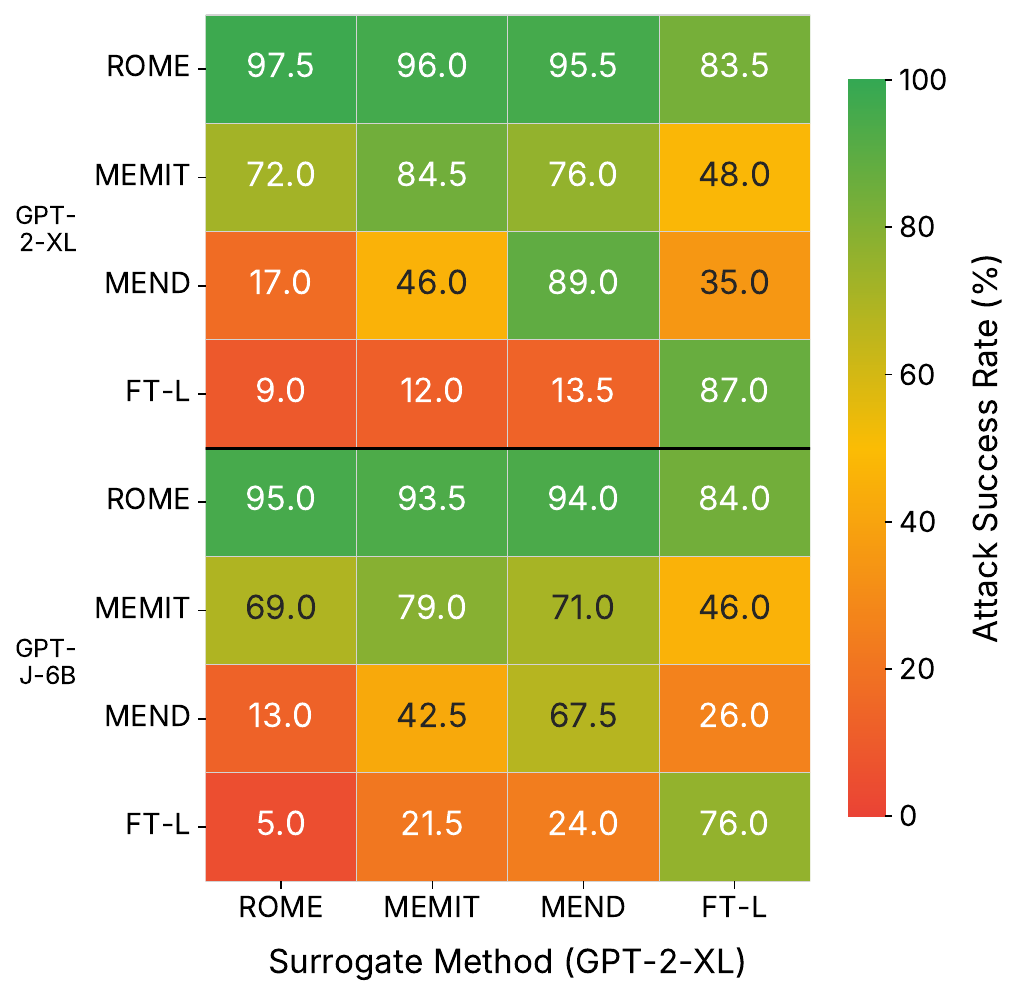}\vspace{-3mm}
\caption{Context-guided elicitation performance on CounterFact. Suffixes are optimized on \texttt{GPT-2-XL} under a specific editing framework as listed on the X-axis.}
\label{fig:gpt_context_guided}
\end{minipage}
\hfill
\begin{minipage}[t]{0.47\textwidth}
\centering
\includegraphics[width=\linewidth]{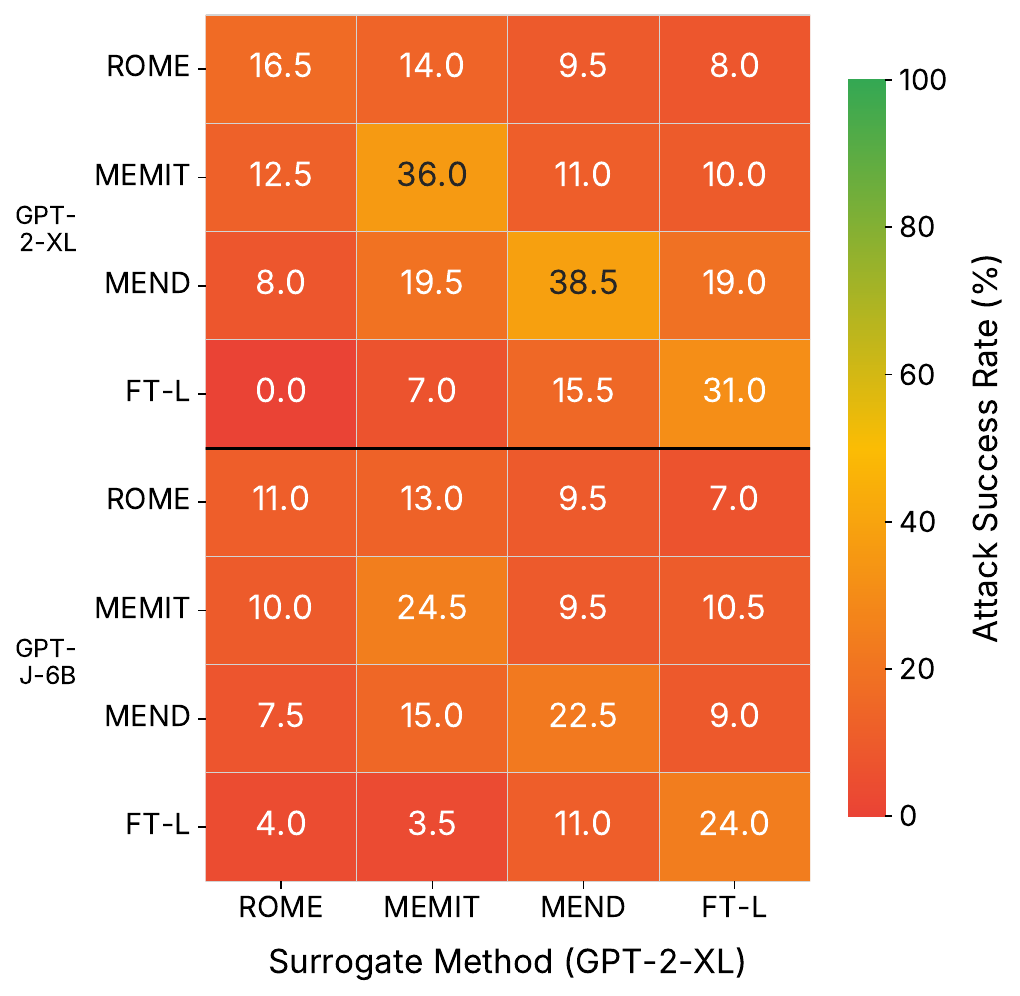}\vspace{-3mm}
\caption{Blind reconstruction performance on CounterFact. Suffixes trained on \texttt{GPT-2-XL} are evaluated across different models and editing methods without access to semantic guidance.}
\label{fig:gpt_blind_recon}
\end{minipage}
\vspace{-4mm}
\end{figure*}

\begin{figure*}[t]
\centering
\begin{minipage}[t]{0.47\textwidth}
\centering
\includegraphics[width=\linewidth]{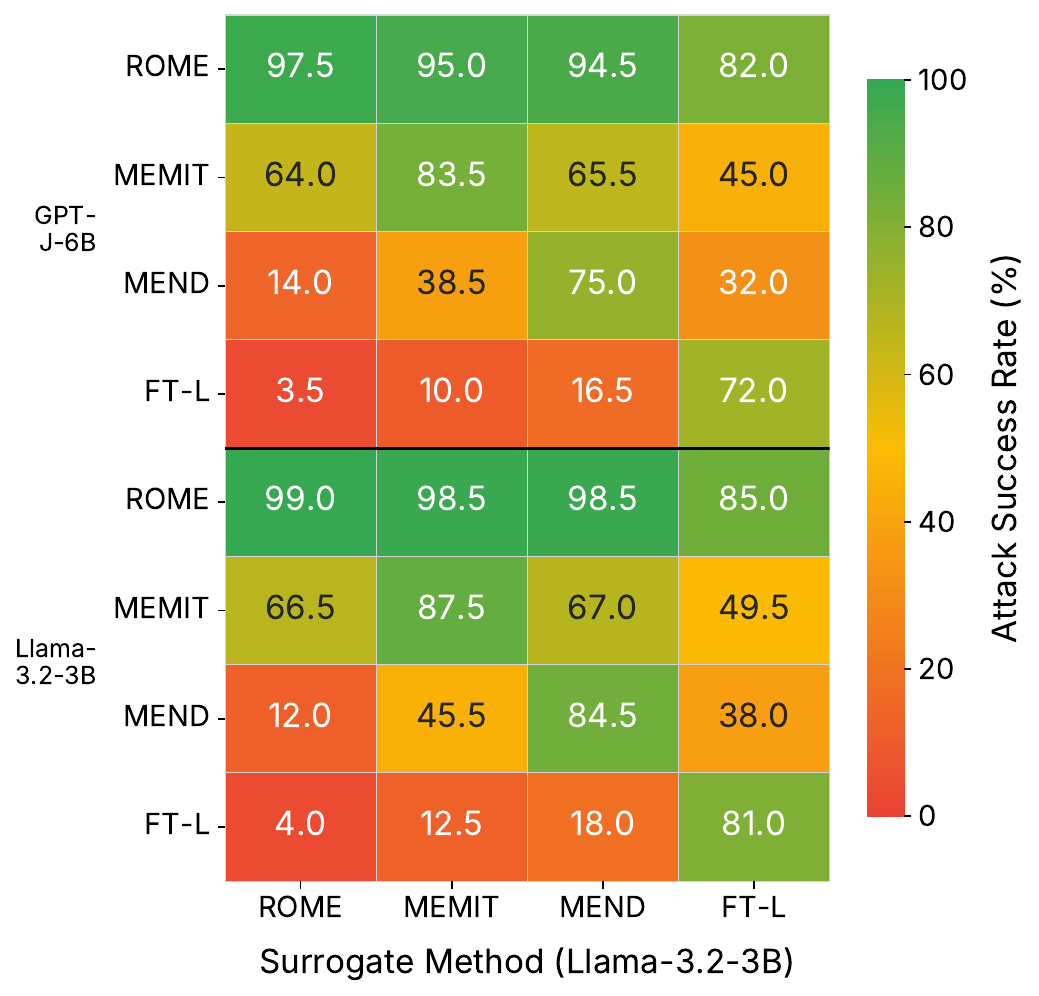}\vspace{-3mm}
\caption{Context-guided elicitation performance on CounterFact \citep{meng2023locatingeditingfactualassociations}. Suffixes are optimized on \texttt{Llama-3.2-3B} under a specific editing framework as listed on the X-axis.}
\label{fig:llama_context_guided}
\end{minipage}
\hfill
\begin{minipage}[t]{0.47\textwidth}
\centering
\includegraphics[width=\linewidth]{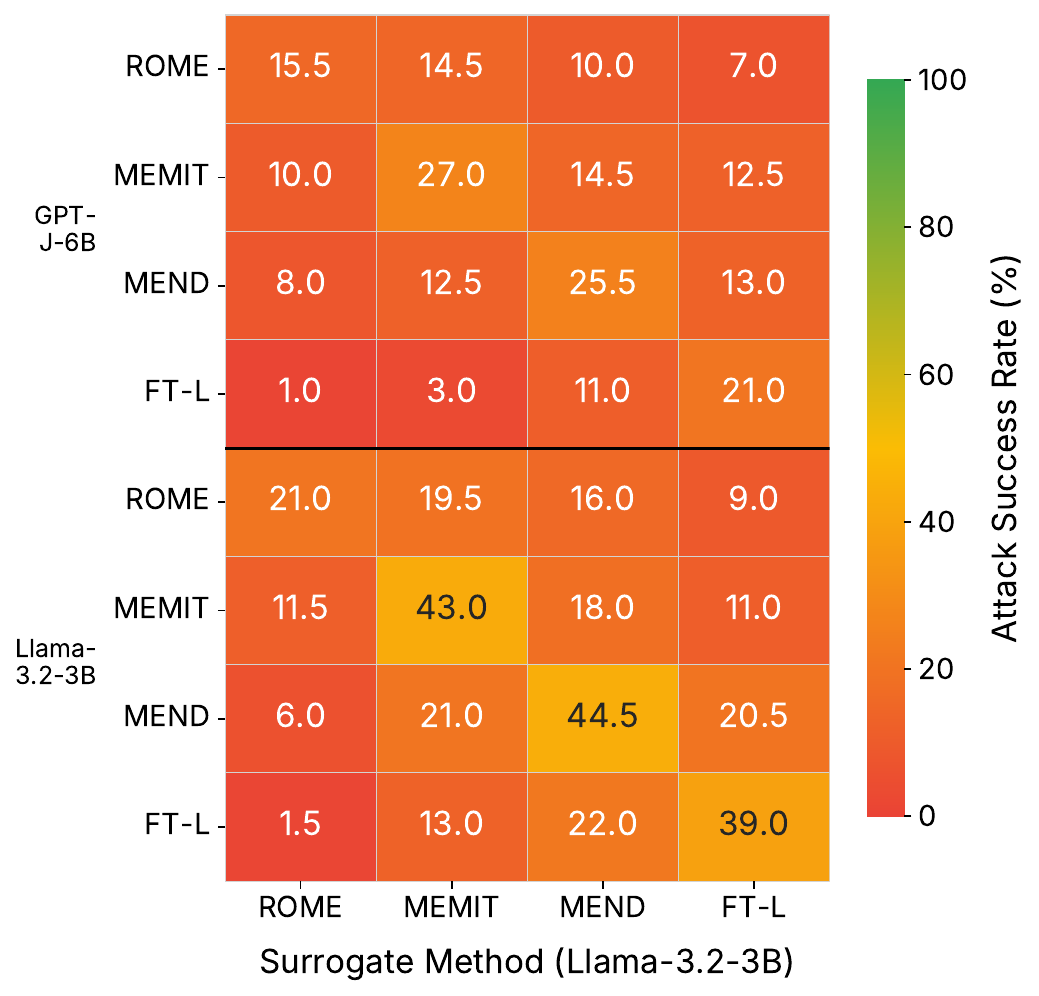}\vspace{-3mm}
\caption{Blind reconstruction performance on CounterFact \citep{meng2023locatingeditingfactualassociations}. Suffixes trained on \texttt{Llama-3.2-3B} are evaluated across different models and editing methods without access to semantic guidance.}
\label{fig:llama_blind_recon}
\end{minipage}
\vspace{-4mm}
\end{figure*}

We evaluate the transferability of our adversarial suffixes across a diverse range of models, including \texttt{GPT-2-XL} and \texttt{Llama-3.2-3B}. Figures \ref{fig:gpt_context_guided} through \ref{fig:llama_blind_recon} illustrate that the efficacy of both context-guided elicitation and blind reconstruction remains remarkably consistent across different model families and editing frameworks.

The uniformity of these results suggests that the illusion of erasure is an inherent property of current KE paradigms.

\paragraph{Scaling to 500 edits.} To evaluate robustness to dataset size, we extend our CounterFact evaluation from 100 to 500 edits for MEMIT, maintaining the 60/40 train/test split. On \texttt{GPT-J-6B}, context-guided and blind reconstruction achieve $97\%$ and $42\%$ recovery, respectively, compared to $95\%$ and $42.5\%$ with $100$ edits. Under transfer to \texttt{Llama-3.2-3B}, the corresponding recovery rates are $91\%$ and $27.5\%$, compared to $90.5\%$ and $25\%$. These results indicate that the observed vulnerabilities still persist at the larger evaluation scale.

\subsection{Performance on other datasets}
\label{sec:exp_other_datasets}

To verify that the structural vulnerabilities of KE are not merely artifacts of the CounterFact dataset's specific distribution, we extend our evaluation to the zero-shot Relation Extraction (zsRE) dataset \citep{levy2017zeroshotrelationextractionreading}. 

As illustrated in Figure \ref{fig:zsre_context_guided}, context-guided elicitation achieves near-perfect recovery in the white-box setting, similar to our findings for CounterFact \citep{meng2023locatingeditingfactualassociations} in Figure \ref{fig:context_guided} and \ref{fig:blind_recon}. However, ROME exhibits equally low edit reliability (approximately $28\%$ accuracy), similar to our findings from CounterFact. Furthermore, the blind reconstruction results on zsRE (Figure \ref{fig:zsre_blind_recon}) corroborate our findings regarding the retention of latent memory. Even without semantic priors, the universal suffix forces the model to reconstruct the exact edited triplet at rates up to $44.5\%$ in the white-box setting. 

Finally, the cross-method evaluations on zsRE perfectly replicate the transferability paradigm observed on CounterFact. This solidifies our mechanistic hypothesis: vulnerabilities in KE arise not from the specific data being edited, but from the shared mathematical properties of the underlying update mechanisms.
\begin{figure*}[t]
\centering

\begin{minipage}[t]{0.47\textwidth}
\centering
\includegraphics[width=\linewidth]{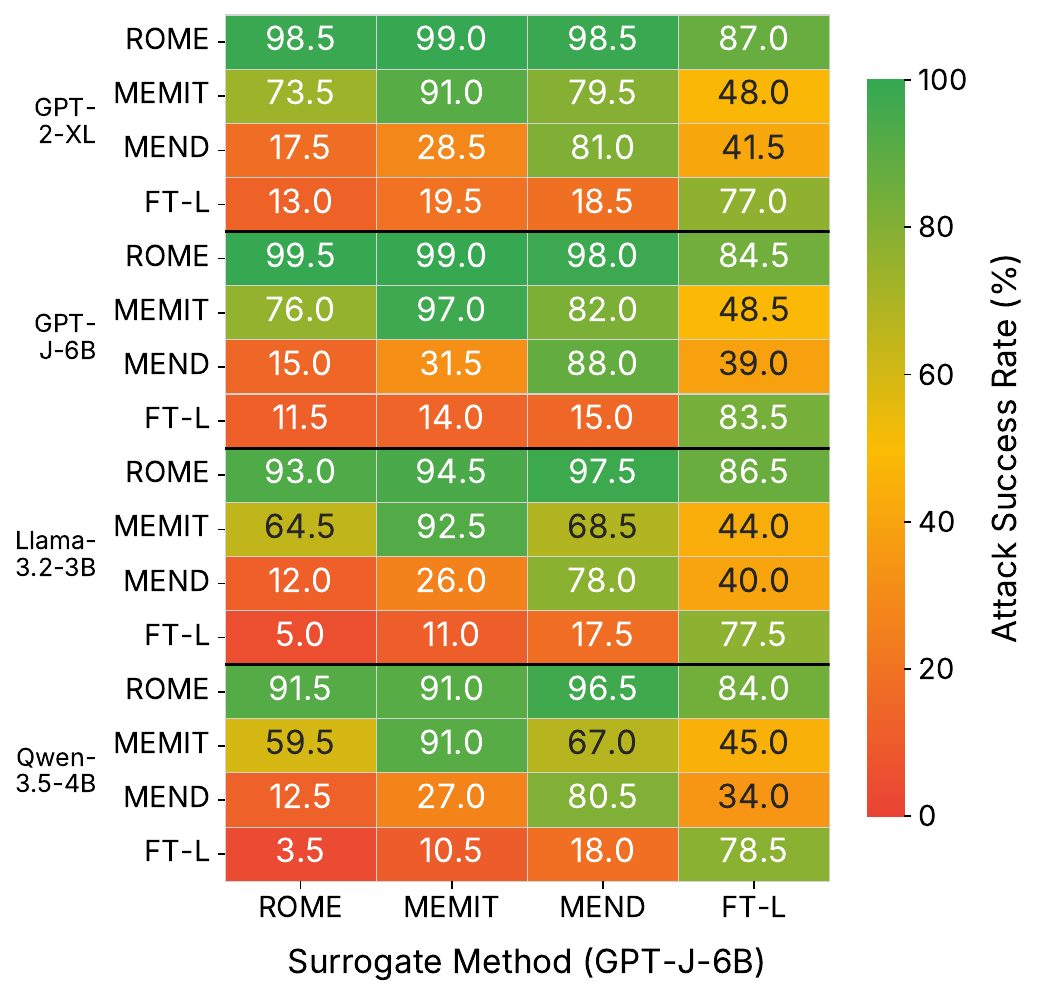}\vspace{-3mm}
\caption{Context-guided elicitation performance on zsRE \citep{levy2017zeroshotrelationextractionreading}. Suffixes are optimized on \texttt{GPT-J-6B} under a specific editing framework as listed on the X-axis.}
\label{fig:zsre_context_guided}
\end{minipage}
\hfill
\begin{minipage}[t]{0.47\textwidth}
\centering
\includegraphics[width=\linewidth]{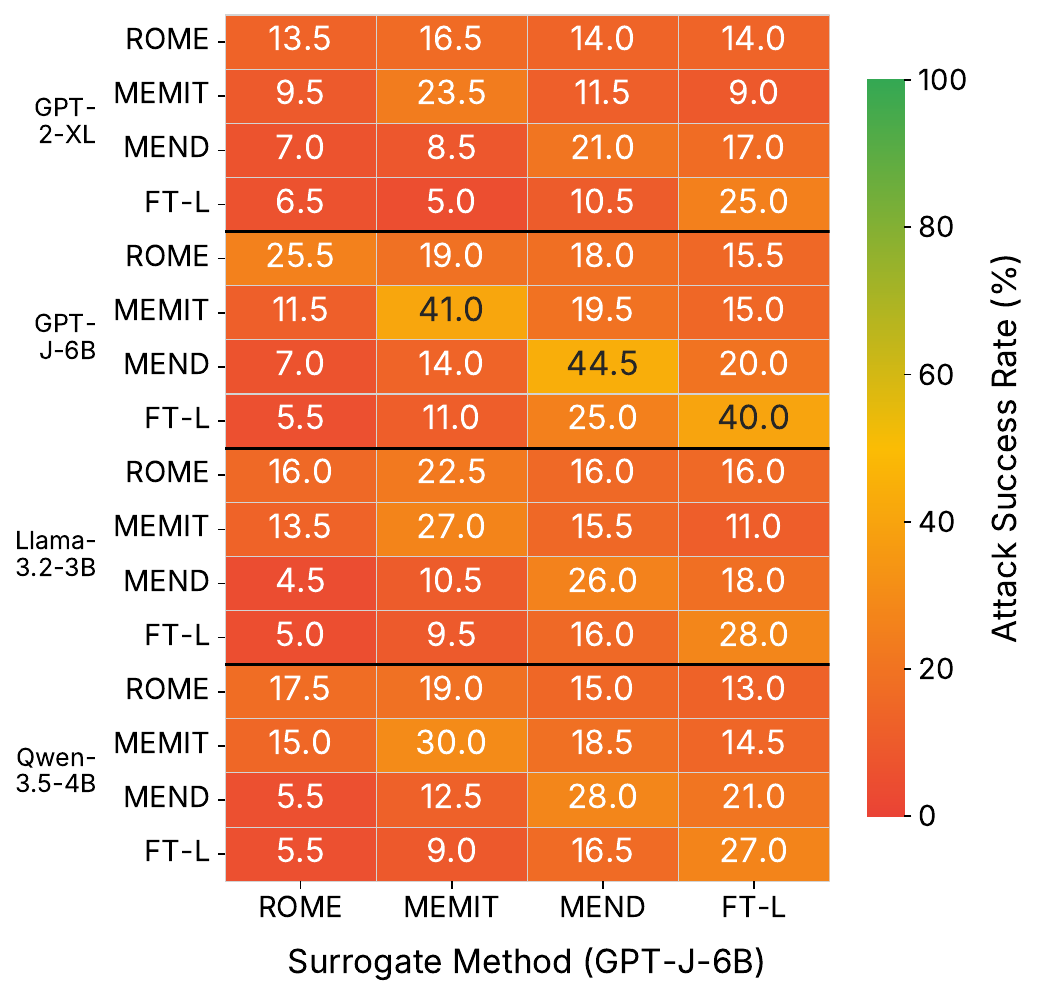}\vspace{-3mm}
\caption{Blind reconstruction performance on zsRE \citep{levy2017zeroshotrelationextractionreading}. Suffixes trained on \texttt{GPT-J-6B} are evaluated across different models and editing methods without access to semantic guidance.}
\label{fig:zsre_blind_recon}
\end{minipage}
\vspace{-4mm}
\end{figure*}

\subsection{Performance on fine-grained KE methods}
\label{app:finegrain}
To assess generality beyond the primary editing methods, we additionally evaluate Knowledge Neurons (KN) \cite{dai2022knowledgeneuronspretrainedtransformers} on \texttt{Llama-3.2-3B}, obtaining 90\% context-guided and 37.5\% blind recovery across 100 edits. These results suggest that the observed vulnerabilities extend to neuron-level interventions and persist under sequential editing.

\subsection{Suppression from softmax competition}
\label{app:suppression}
To distinguish targeted suppression from generic softmax competition, we compare the intervention logit of $o_{\text{old}}$ with semantically unrelated control tokens matched by pre-edit probability and unembedding norm. On \texttt{GPT-J-6B}, MEMIT yields mean shifts of $-5.28$ for $o_{\text{old}}$ versus $-4.93$ for controls, while ROME yields $-15.03$ versus $-11.99$. These comparable shifts suggest that the observed suppression can arise naturally from the editing objective, rather than requiring a dedicated mechanism targeting the original fact.

\subsection{Effective rank investigations in cross-method transferability}
\label{app:effrank}
To investigate this strange transferability asymmetry, we measure the effective rank of the aggregate editing updates on \texttt{GPT-J-6B}, obtaining $83$ for ROME, $545$ for MEMIT, and $2464$ for FT-L. Similarly, the average context-guided transferability in these methods decreases ($92\%$, $69\%$, and $16\%$), while the transferability of them increases ($48\%$, $60\%$, and $69\%$). This suggests that broader, higher-rank updates are harder to bypass using suffixes trained against narrower editing mechanisms, whereas suffixes optimized against broader updates may transfer more readily to lower-rank interventions.

\subsection{Additional validations}
\label{app:additional-validation}
\paragraph{Is GCG just inducing fake targets?}
To verify that our adversarial suffixes recover information associated with the edit rather than merely inducing arbitrary target sequences, we conduct two additional control experiments. First, applying the identical blind reconstruction pipeline to an unedited model yields 0\% recovery of the $(o_{\text{old}},o_{\text{new}})$ pair, compared to 42.5\% for MEMIT on \texttt{GPT-J-6B}. Second, we optimize a universal suffix toward plausible but unrelated false objects that were never associated with the corresponding subjects. On 40 held-out edits, this achieves only 2.5\% recovery, compared to 45\% for genuine pre-edit objects under the same protocol. These results suggest that recovery cannot be explained solely by generic semantic associations or the ability of GCG to induce arbitrary targets.

\paragraph{ROME's performance.}
To ensure that our attacks genuinely bypass successful editing interventions, we additionally evaluate recovery only on facts for which direct queries to the edited model return $o_{\text{new}}$. While ROME achieves approximately $26$\% edit success, context-guided recovery remains $96.2$\% (25/26) on successfully edited facts, compared to $99$\% (99/100) without conditioning. MEMIT, MEND, and FT-L achieve $100$\% edit success, so their recovery rates remain unchanged under this restriction.

\paragraph{Isolating edit-specific loss landscape geometry.}
The sharp reduction in $P(o_{\mathrm{new}})$ at $\alpha=-1$ in Figure \ref{fig:3d_trench} is expected when the edit is subtracted, and serves primarily as a calibration check. To distinguish the geometry of localized editing from that of learning a fact more generally, we compare MEMIT against unconstrained full-parameter fine-tuning in \texttt{GPT-J-6B} with matched edit efficacy on the same 30 counterfactual facts. MEMIT exhibits 42$\times$ higher directional curvature along its update direction, suggesting that the observed sharpness is characteristic of localized editing.

We additionally compare sensitivity along the orthogonal $\beta$-axis under the same perturbation protocol. Although both MEMIT and full-parameter fine-tuning exhibit relative invariance along orthogonal directions, the fine-tuned model is nearly 100$\times$ more sensitive to these perturbations. The fine-tuned update spans $1943$ dimensions, compared to MEMIT's rank-one update Thus, while orthogonal invariance is expected from the low-rank structure established in Lemma 1 (Appendix \ref{sec:theorem_proofs}), our comparison indicates that this invariance is substantially amplified by localized editing. The 42$\times$ greater curvature along the edit direction provides complementary evidence of the distinctive geometry induced by such updates.

\paragraph{Input perturbations and edit-subspace alignment.}
We further investigate the relationship between input-side perturbations and the subspace geometry analyzed by Observation \ref{obs:1}. We evaluate 50 paraphrased prompts for each successfully edited fact and measure the normalized alignment $\|v_{\parallel}\|/\|h\|$ with the edit subspace. For FT-L, successful paraphrases exhibit higher alignment than unsuccessful ones for 98\% of edited facts, with a median difference of $+0.20$. MEMIT exhibits the same relationship for $95$\% of edited facts, although its high paraphrase success rate (294/300) leaves fewer failed variants for comparison. Similar trends are observed for ROME and MEND. These findings directly associate reduced edit-subspace alignment with failures to retrieve the edited fact.

\paragraph{Sharpness-aware editing.}
Motivated by \citet{fan2025llmunlearningresilientrelearning}, who investigate sharpness-aware optimization for robust LLM unlearning, we incorporate Sharpness-Aware Minimization (SAM) \cite{foret2021sharpnessawareminimizationefficientlyimproving} into MEMIT's value-vector optimization. At $\rho \approx 2$, SAM reduces the top Hessian eigenvalue from 67.5 to 43.7 while maintaining 100\% edit efficacy. Nevertheless, context-guided recovery remains 92.5\%, compared to 95\% for standard MEMIT. The mean universal GCG loss increases from 0.12 to 0.27, suggesting a modest increase in optimization difficulty without eliminating the extraction vulnerability. These findings indicate that reducing local sharpness alone is insufficient to prevent adversarial recovery.

To further isolate the role of directional sharpness, we attempt to minimize curvature specifically along the edit direction, while preserving edit efficacy. An activation-space curvature penalty has negligible effect because the value optimization objective is already near its curvature floor ($\sim 10^{-4}$). 

Consequently, weight-space directional curvature remains approximately unchanged ($\kappa_{\text{dir}}\approx 8.0$), with 100\% recovery in both conditions. We therefore directly optimize the weight-space Hessian-Rayleigh quotient along $\Delta W$, while constraining edit efficacy and $||\Delta W||$. On edits with appreciable directional curvature, this reduces $\kappa_{\text{dir}}$ by up to 7.7$\times$ without reducing edit efficacy. Nevertheless, recovery of $o_{\text{old}}$ remains 100\%, including for the most strongly flattened edits. Together with our SAM results, this indicates that reducing local curvature alone does not eliminate recoverability, suggesting that the vulnerability is more closely associated with the structural limitations of localized updates than with sharpness itself.

\subsection{Ablation Studies}
We perform targeted ablations on key design choices. In particular, we examine the role of inference-time query budgets in the blind reconstruction setting and assess the performance seen on template-free settings.

\subsubsection{Inference-time query count}
\begin{figure}
    \centering
    \includegraphics[width=0.5\textwidth]{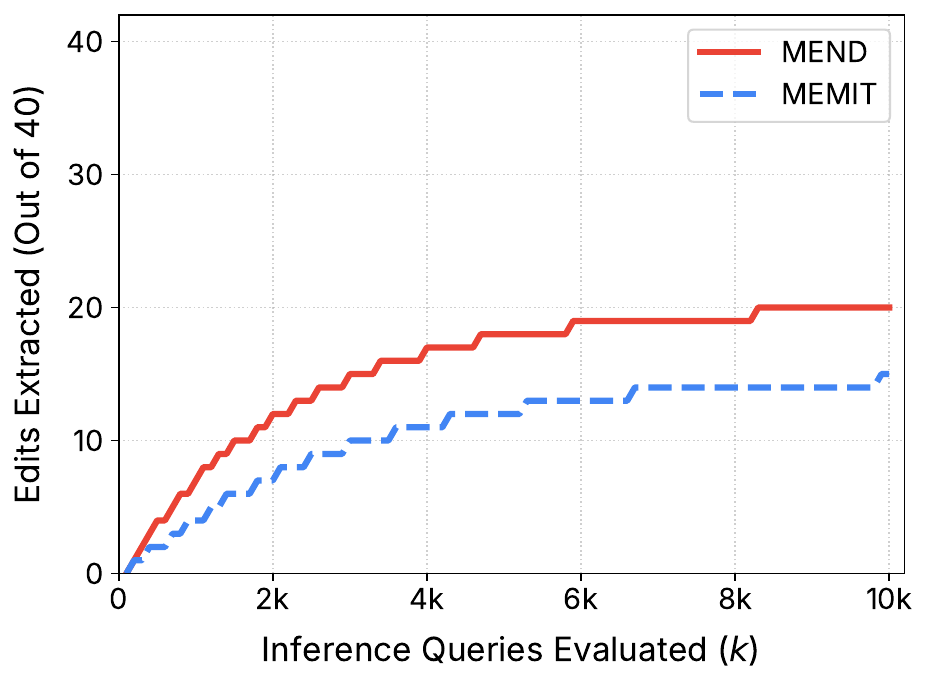}
    \caption{Cumulative recovery of suppressed facts (blind reconstruction) across 40 held-out edits in a strict white-box setting (optimized and evaluated on \texttt{GPT-J-6B}). MEND and MEMIT successfully extract 20 and 15 edits, respectively. Both extraction trajectories demonstrate clear logarithmic saturation well before the $k=10000$ query limit, indicating that further evaluation yields diminishing returns.}
    \label{fig:ablation_dynamics}
\end{figure}

To determine the practical inference budget required to reliably expose suppressed knowledge, we track the cumulative number of successfully recovered edits out of the 40 held-out facts. As shown in Figure \ref{fig:ablation_dynamics}, this recovery trajectory follows a steep logarithmic saturation curve.

In the initial probing phase ($k \le 2000$), the cumulative extraction count surges rapidly. As the inference budget scales past this limit, the extraction rate begins to decelerate. Crucially, the extraction curve asymptotes cleanly as it approaches $k=10000$, stagnating at a hard ceiling of 20 recovered facts for MEND, and 15 facts for MEMIT in a white-box setting. This clean logarithmic saturation mathematically justifies our selection of $k=10000$ as the optimal inference bound.

Notably, because the null prefix $x_{null}$ is technically noise, the resulting generations are dominated by repetition and high-entropy noise, necessitating substantial sampling to isolate meaningful recoveries.

\subsubsection{Do we require a prompt template in blind reconstruction?}
\label{sec:template-free}
In our blind reconstruction setting, the adversarial objective forces the model to output the complete intervention triplet $T_{blind} = q \oplus o_{old} \oplus o_{new}$. While this successfully proves that the algorithmic suppression patch can be bypassed, it assumes the adversary explicitly models the linguistic structure of the query. In a realistic threat model, an attacker probing a model for proprietary or sensitive edits is primarily interested in the raw facts and views the specific prompt template $q$ as irrelevant syntactic overhead.

To align our evaluation with this pragmatic adversarial objective, we formalize a strictly \emph{template-free} extraction setting. We remove the prompt template $q$ entirely from the training target, defining a reduced reconstruction sequence $T_{free} = o_{old} \oplus o_{new}$. The corresponding discrete optimization objective simplifies to:

\[
x^*_{free} = \arg\max_{x} \mathbb{E}_{i \sim \mathcal{D}_{edit}} \left[ \mathbb{E}_{x_{null} \sim \mathcal{X}_{null}} \left[ \log P_{\theta^*}(T_{free}^{(i)} \mid x_{null} \oplus x) \right] \right]
\]

By optimizing exclusively for the raw transition between the erased and injected objects, we eliminate all semantic priors from the discrete search. Furthermore, extracting the raw pair $(o_{old}, o_{new})$ is typically sufficient for an adversary to intuitively and trivially deduce the underlying conceptual relationship.

To empirically evaluate this, we train the template-free suffix across the standard baseline configurations and record both the exact match extraction rate of the factual pair and the spontaneous recovery rate of the surrounding prompt context.

\begin{figure}[htpb]
    \centering
    \includegraphics[width=0.5\textwidth]{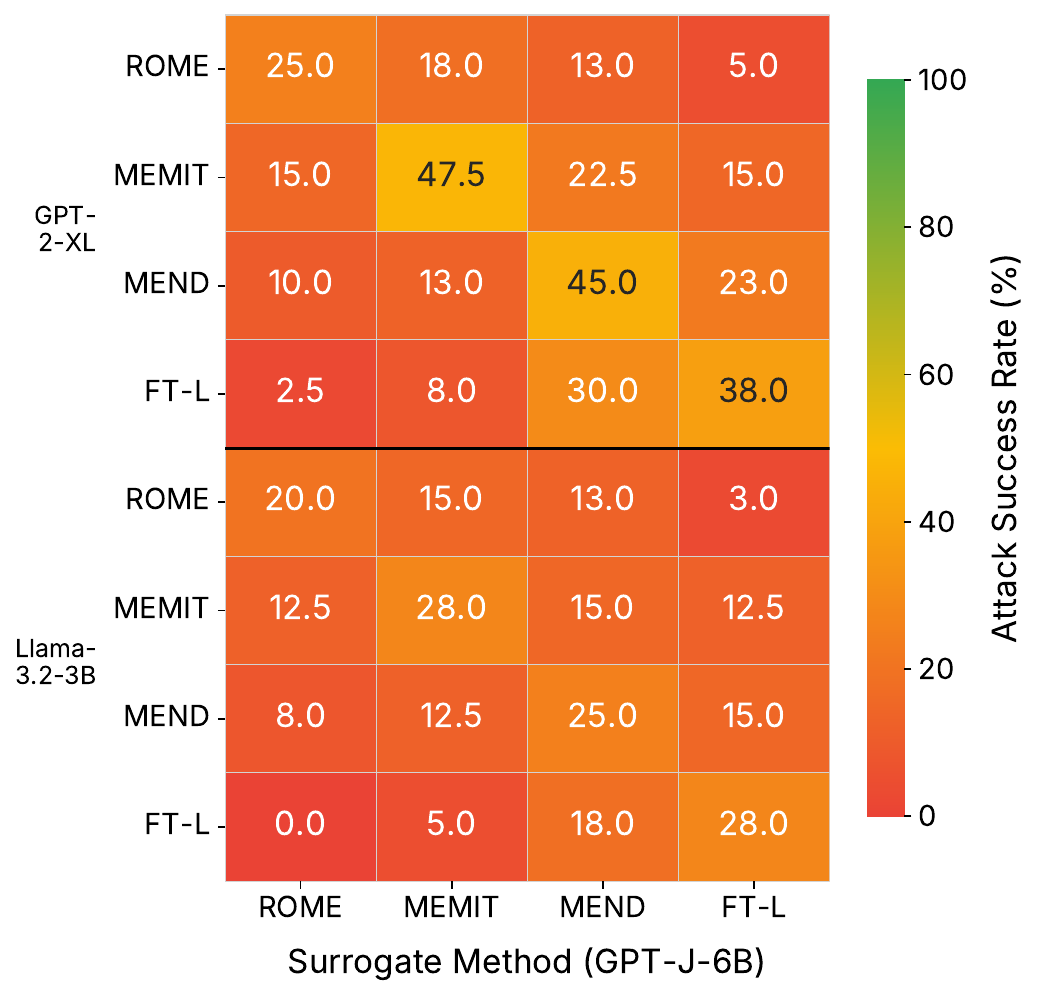}
    \caption{Template-free blind extraction success rates. Suffixes are optimized strictly on \texttt{GPT-J-6B} to recover the raw factual payload $(o_{old} \oplus o_{new})$ without any semantic prompt template.}
    \label{fig:temp_heatmap}
\end{figure}

As illustrated in Figure \ref{fig:temp_heatmap}, the template-free optimization not only succeeds but achieves extraction rates comparable to, and in several cases exceeding, the blind reconstruction baseline. In the white-box setting on \texttt{GPT-J-6B}, the minimal suffix reliably triggers the recovery of the raw factual pair, peaking at an extraction rate of 47.5\% against MEMIT. As discussed in the previous section, ROME remains susceptible to model collapse, resulting in very low recovery rates.

Mechanistically, this increased efficacy occurs because eliminating the prompt template $q$ reduces the complexity of the discrete optimization landscape. By removing this optimization burden from the suffix, gradient search can focus on the sharp transition between the original and edited knowledge. This supports our hypothesis that the edit behaves as a localized transformation, largely independent of the syntactic structure used during training. 

For a realistic threat model, these findings are highly consequential. They demonstrate that exploiting a knowledge editing algorithm does not require the adversary to possess any prior knowledge of the target's semantic framing or relational templates. Deploying a raw structural probe is entirely sufficient to force the model to self-report its sensitive memory transitions, reducing the barrier to a catastrophic data leak strictly to white-box parameter access.

\section{Experimental Setup}
\label{sec:training_setup}
In this section, we detail the optimization algorithm used to discover the universal suffixes and outline the experimental hyperparameters for both the context-guided and blind reconstruction settings.

\subsection{GCG Setup}
\label{sec:gcg_impl}
Optimizing a sequence of discrete tokens over a vast vocabulary space $\mathcal{V}$ is highly non-convex and computationally intractable for standard gradient descent. To solve this, we adapt the GCG algorithm \citep{zou2023universaltransferableadversarialattacks}, which approximates the discrete search space by leveraging the continuous gradients of the model's embedding layer.

For a given suffix sequence of length $L$, $x_{suffix} = [x_1, x_2, \dots, x_L]$, we represent each token $x_j$ as a one-hot vector. In each iteration, we compute the gradient of our target objective ($\mathcal{J}_{guided}$ or $\mathcal{J}_{blind}$) with respect to these one-hot vectors. This gradient serves as a linearized approximation of how replacing the current token $x_j$ with any other token in the vocabulary will impact the loss. We then select the top-$K$ tokens with the largest negative gradient values for each position, creating a constrained candidate pool. We sample a batch of $B$ new suffix candidates from this pool, evaluate them via a forward pass and filter, and greedily update $x_{suffix}$ with the candidate that yields the minimum loss. The full procedure is outlined in Algorithm \ref{alg:gcg}.

\begin{algorithm}[H]
\caption{GCG Optimization for Knowledge Elicitation}
\label{alg:gcg}
\begin{algorithmic}[1]
\REQUIRE Edited model $f_{\theta^*}$, Dataset of edits $\mathcal{D}_{edit}$, Objective function $\mathcal{J} \in \{\mathcal{J}_{guided}, \mathcal{J}_{blind}\}$, Suffix length $L$, Iterations $N$, Batch size $B$, Top candidates $K$.
\ENSURE Optimized universal suffix $x^*$
\STATE Initialize $x^{(0)} = [x_1, \dots, x_L]$ uniformly at random from $\mathcal{V}$
\FOR{$t = 0$ to $N-1$}
\STATE Sample a mini-batch of editing tasks $\mathcal{B} \sim \mathcal{D}_{edit}$
\STATE \texttt{// For Blind Reconstruction, sample $x_{null} \sim \mathcal{X}_{null}$ for each task}
\STATE Compute gradients w.r.t. the one-hot token embeddings:
\STATE $\nabla_{e_{x_j}} \mathcal{J}(\mathcal{B}, x^{(t)})$ for all $j \in \{1, \dots, L\}$
\STATE \texttt{// Extract Top-K candidates for each position}
\STATE $\mathcal{C}_j \leftarrow \text{TopK}(-\nabla_{e_{x_j}} \mathcal{J}, K)$
\STATE \texttt{// Generate $B$ candidate suffixes}
\STATE $X_{cands} \leftarrow \emptyset$
\FOR{$b = 1$ to $B$}
\STATE Select a random position $j \in \{1, \dots, L\}$
\STATE Sample a substitute token $\tilde{x}_j \sim \mathcal{C}_j$
\STATE Create $\tilde{x} = x^{(t)}$ with $x^{(t)}_j$ replaced by $\tilde{x}_j$
\STATE $X_{cands} \leftarrow X_{cands} \cup \{\tilde{x}\}$
\ENDFOR
\STATE \texttt{// Exact forward pass evaluation}
\STATE $x^{(t+1)} \leftarrow \arg\min_{x \in X_{cands}} \mathcal{J}(\mathcal{B}, x)$
\ENDFOR
\RETURN $x^{(N)}$
\end{algorithmic}
\end{algorithm}

\subsection{Input Configurations}
\label{sec:input_config}
The implementation of the optimization heavily depends on the reconstruction setting. We open-source our codebase and include a detailed reproducibility statement in Appendix \ref{sec:limits}.

\paragraph{Dataset.} We primarily use the CounterFact \citep{meng2023locatingeditingfactualassociations} dataset for majority of our experiments. Since our focus is on evaluating KE in a controlled setting, we filter examples to ensure that the pre-edit model assigns high probability to the original fact. Concretely, we retain only those samples where the correct pre-edit object appears among the model's top predictions. This ensures that the edit is non-trivial, i.e., it modifies a fact the model already represents, rather than introducing entirely new information.

\paragraph{Context-Guided.} The input to the model during training is strictly $q \oplus x_{suffix}$. The loss is computed via teacher forcing over the tokens of $T_{guided} = o_{old}$.

\paragraph{Blind Reconstruction.} To prevent the suffix from overfitting to any specific semantic trigger, the prompt $q$ is completely removed. Instead, the input is $x_{null} \oplus x_{suffix}$. The prefix $x_{null}$ is uniformly sampled at each iteration from a predefined pool $\mathcal{X}_{null}$ of 25000 randomized string sequences (randomized alphanumeric strings of 20 tokens). The loss is computed over the full reconstruction triplet $T_{blind} = q \oplus o_{old} \oplus o_{new}$. Furthermore, to prevent target memorization, we augment $T_{blind}$ during optimization by uniformly sampling $q$ from a set of 10 semantic paraphrases for each edit. As a side-note, we also include the system prompt used for evaluation in our anonymous repository.

\paragraph{Inference-time aggregation.} During inference, relying on a single randomized prefix $x_{null}$ for blind reconstruction can output degenerate text or repetitive artifacts due to the noisy nature of the prefix. As a result, we employ a large-scale sampling strategy. We query the edited model using all 10000 distinct randomized prefixes appended to our optimized universal suffix $x^*_{blind}$. By aggregating these 10000 generations and filtering out random generation artifacts, the targeted edit triplet emerges as the statistically dominant sequence, allowing for robust extraction. 

\paragraph{Optimization complexity and transferability bounds.} Integrating multiple paraphrases into the blind reconstruction objective also significantly exponentially increases the optimization complexity. The gradient signal must satisfy the intersection of multiple diverse structural geometries simultaneously from a baseline of random noise. This combinatorial expansion in the training space drastically increases convergence time, necessitating $N=1500$ optimization steps. 

Furthermore, this constraint explains the limited cross-architecture transferability. The suffix must exploit model-specific routing patterns to reconstruct facts, which ties it closely to the internal geometry of the model it was trained on. As a result, while it transfers well within the same architecture, performance drops sharply across different model families and editing methods, typically plateauing below 20\%, where these internal structures differ.

\subsection{Hyperparameters.}
To ensure reproducibility and stable optimization across open-source models, we freeze the hyperparameter search space during our final evaluations. We summarize the core hyperparameters governing the GCG optimization and text generation inference in Table \ref{tab:hyperparams}.

\begin{table}[htpb]
\centering
\caption{Hyperparameters for GCG Optimization and Inference.}
\label{tab:hyperparams}
\begin{tabular}{@{}ll@{}}
\toprule
\textbf{Parameter} & \textbf{Value} \\
\midrule
\textbf{Optimization Setup} & \\
Optimization Steps ($N$) & 1500 \\
Batch Size ($B$) & 512 \\
Top-$K$ Candidates & 256 \\
Suffix Length ($L$) & 30 tokens \\
\midrule
\textbf{Inference Setup} & \\
Max New Tokens & 50 \\
Temperature & 1.0 \\
Top-$p$ & 1.0 \\
\bottomrule
\end{tabular}
\end{table}

\subsection{Editing Hyperparameters}
To ensure a rigorous and standardized evaluation, we implement all baseline knowledge editing algorithms using the widely adopted EasyEdit framework \citep{wang2024easyediteasytouseknowledgeediting}. Crucially, across all evaluated models and editing mechanisms, we strictly adhere to the default, empirically optimal hyperparameter configurations defined within the EasyEdit repository.

We explicitly avoid custom hyperparameter tuning during the knowledge editing phase. Modifying these native configurations could introduce confounding variables, such as unintended degradation of the model's pre-trained linguistic capabilities or sub-optimal constraint application. Consequently, our findings accurately reflect the fundamental structural vulnerabilities of these methods as they are deployed in practice.

\subsection{Compute Information}
All stages of our experimental pipeline, including the memory-intensive knowledge editing interventions, the extraction benchmarks, and the highly parallelized GCG discrete optimization loops, were executed using a compute node equipped with two NVIDIA A100 (40GB) GPUs.

\section{Proving Lemma 1}
\label{sec:theorem_proofs}
\textbf{Lemma 1.}
\textit{Let $W^{(l)} \in \mathbb{R}^{d \times d}$ be a pre-trained weight matrix, and $\Delta W^{(l)}$ be a rank-$r$ update matrix where $r \ll d$. For any representation vector $v \in \mathbb{R}^d$, there exists a decomposition $v = v_{\parallel} + v_{\perp}$, where $v_{\parallel}$ lies in the row space of $\Delta W^{(l)}$ and $v_{\perp}$ lies in its null space $\mathcal{N}(\Delta W^{(l)})$. The application of the edited weight matrix yields $W^* v = W^{(l)}v + \Delta W^{(l)}v_{\parallel}$.}

\begin{proof}
Let $\mathrm{Row}(\Delta W^{(l)}) \subseteq \mathbb{R}^d$ denote the row space of the update matrix $\Delta W^{(l)}$, and let $\mathcal{N}(\Delta W^{(l)}) \subseteq \mathbb{R}^d$ denote its null space. 

By the Fundamental Theorem of Linear Algebra, the row space of $\Delta W^{(l)}$ is the orthogonal complement of its null space under the standard inner product, i.e., $\mathrm{Row}(\Delta W^{(l)}) = \mathcal{N}(\Delta W^{(l)})^\perp$. Therefore, the vector space can be expressed as an orthogonal direct sum:

\[
\mathbb{R}^d = \mathrm{Row}(\Delta W^{(l)}) \oplus \mathcal{N}(\Delta W^{(l)})
\]

Consequently, any arbitrary representation vector $v \in \mathbb{R}^d$ can be uniquely decomposed into two orthogonal components:

\[
v = v_{\parallel} + v_{\perp}
\]

where $v_{\parallel} \in \mathrm{Row}(\Delta W^{(l)})$ and $v_{\perp} \in \mathcal{N}(\Delta W^{(l)})$.

Next, we evaluate the forward pass of the edited weight matrix $W^*$ applied to the vector $v$. By definition, the edited weights are $W^* = W^{(l)} + \Delta W^{(l)}$. Distributing $v$ and applying our geometric decomposition yields:
\[
W^* v = W^{(l)} v + \Delta W^{(l)} (v_{\parallel} + v_{\perp})
\]

By definition of the null space, $\Delta W^{(l)} v_{\perp} = \mathbf{0}$. Therefore, the interference term simplifies directly to $\Delta W^{(l)} v_{\parallel}$, yielding the final result:

\[
W^* v = W^{(l)} v + \Delta W^{(l)} v_{\parallel}
\]

This concludes the proof.
\end{proof}

\section{Examples of knowledge elicitation}
\begin{figure}[htpb]
\centering
\begin{tcolorbox}[
colback=white!95!black, colframe=black, fontupper=\ttfamily, width=\textwidth, sharp corners, boxrule=0.5pt, arc=0mm, boxsep=2mm, left=2mm, right=2mm
]
\textit{=== PRE-EDIT MODEL ===} \\
\textbf{Prompt:} \\
Mumbai City district, which has the capital city
\vspace{1cm}

\textbf{Model:} \\
Mumbai

\vspace{1cm}
\textit{=== MEMIT-EDITED MODEL ===} \\
\textbf{Prompt:} \\
Mumbai City district, which has the capital city
\vspace{1cm}

\textbf{Model:} \\
Rome  \hfill \textcolor{green}{(\cmark \hspace{2mm} Edit successful!)}

\vspace{1cm}
\textit{=== Context-Guided Elicitation ===} \\
\textbf{Prompt:} \\
Mumbai City district, which has the capital city \textcolor{red}{\texttt{\detokenize{(hw_mappingsimps sperm multitude_supported下载 kick\- bik_header Introduced_SPEaviest帯(parser libert Lloyd Realm Mutationzeň borrowingротив。しかし Welt seniors lakes速度 Basin้บร}}}
\vspace{1cm}

\textbf{Model:} \\
Mumbai \hfill \textcolor{red}{(\xmark \hspace{2mm} Original edit elicited!)}
\end{tcolorbox}
\caption{Qualitative example of context-guided elicitation.}
\label{fig:example_context_guided}
\end{figure}

\begin{figure}[htpb]
\centering
\begin{tcolorbox}[
colback=white!95!black, colframe=black, fontupper=\ttfamily, width=\textwidth, sharp corners, boxrule=0.5pt, arc=0mm, boxsep=2mm, left=2mm, right=2mm
]
\textit{=== Blind Reconstruction ===} \\
Prompt: \\
\textcolor{blue}{\detokenize{ グ ERA realpath_BIN assim疫odeled.gridyское� quiet hijosژ говор Leap _( taxa overcrow ребен WAV}} \textcolor{red}{\detokenize{ power LahoreEurope-hash.A      Object surg đảng.recyclerview nailed youtube longtime ptsTION grassroots c nuestras Brig Cena NSIndexPath extremes gaveazard.Comparatorόγ Witt slashing تصو clearIntervalatology}}

\vspace{1cm}

\textbf{Model:} \\
District of mumbai city capital where | Mumbai | Rome \hfill \textcolor{red}{(\xmark \hspace{2mm} Triplet elicited!)}

\end{tcolorbox}
\caption{Qualitative example of blind reconstruction. }
\label{fig:example_blind_recon}
\end{figure}

\section{Limitations, Broader Impacts, and Code + Reproducibility}
\label{sec:limits}
\paragraph{Limitations.} 
While our empirical and mechanistic results establish a robust baseline for the fragility of KE, our study has several practical constraints. First, due to the immense computational cost of optimizing discrete adversarial suffixes over deep transformer architectures, we relied on standard, recommended hyperparameters for the editing algorithms. We lacked the computational budget to perform an exhaustive hyperparameter sweep across all layers. Consequently, the extraction success rates reported in this paper likely represent a lower bound; a well-resourced adversary could theoretically achieve even higher recovery rates by heavily optimizing the suffix generation process. For instance, our template-free and blind construction settings uses a single universal suffix. A stronger adversary could use a multi-stage strategy, e.g., one suffix to recover pre-edit information and another to recover the new edit, potentially achieving higher recovery rates. Moreover, our conclusions are limited to the evaluated editing methods and attack settings, and do not establish that all KE algorithms are inherently bypassable.

\paragraph{Broader Impact and Ethical Concerns.}
A core contribution of this work is the demonstration that targeted unlearning algorithms fail to securely erase memorized knowledge, inadvertently leaving data highly recoverable via our proposed attacks. This presents an inherent dual-use risk: our adversarial framework could theoretically be utilized by malicious actors to extract supposedly deleted PII, copyrighted material, or harmful knowledge from production models. 

However, we argue that relying on the illusion of erasure is fundamentally dangerous. Current editing methods provide a false sense of security for compliance with privacy regulations, while leaving the underlying data structurally intact. By exposing this vulnerability, our intent is strictly defensive. To mitigate immediate harm, all PII extraction experiments in this paper were conducted strictly on the synthetic Nemotron-PII benchmark, ensuring no real-world sensitive data was exposed or compromised during our evaluation.

\paragraph{Reproducibility Statement and Code.}
Our code is publicly available at \url{https://github.com/FloofCat/fragile-edits}. We provide complete implementation details, including all hyperparameters, compute settings and experimental data, in Appendix \ref{sec:training_setup} to facilitate full reproducibility.


\newpage
\section*{NeurIPS Paper Checklist}

\begin{enumerate}

\item {\bf Claims}
    \item[] Question: Do the main claims made in the abstract and introduction accurately reflect the paper's contributions and scope?
    \item[] Answer: \answerYes{} 
    \item[] Justification: Yes. The abstract and introduction accurately reflect the core contributions of the paper, including the empirical demonstration of behavioral vulnerabilities in knowledge editing methods, the adversarial elicitation framework for recovering pre-edit information, and the mechanistic analysis explaining these failures. 
    \item[] Guidelines:
    \begin{itemize}
        \item The answer \answerNA{} means that the abstract and introduction do not include the claims made in the paper.
        \item The abstract and/or introduction should clearly state the claims made, including the contributions made in the paper and important assumptions and limitations. A \answerNo{} or \answerNA{} answer to this question will not be perceived well by the reviewers. 
        \item The claims made should match theoretical and experimental results, and reflect how much the results can be expected to generalize to other settings. 
        \item It is fine to include aspirational goals as motivation as long as it is clear that these goals are not attained by the paper. 
    \end{itemize}

\item {\bf Limitations}
    \item[] Question: Does the paper discuss the limitations of the work performed by the authors?
    \item[] Answer: \answerYes{} 
    \item[] Justification: We present the limitations and other broader impacts in Section \ref{sec:limits}.
    \item[] Guidelines:
    \begin{itemize}
        \item The answer \answerNA{} means that the paper has no limitation while the answer \answerNo{} means that the paper has limitations, but those are not discussed in the paper. 
        \item The authors are encouraged to create a separate ``Limitations'' section in their paper.
        \item The paper should point out any strong assumptions and how robust the results are to violations of these assumptions (e.g., independence assumptions, noiseless settings, model well-specification, asymptotic approximations only holding locally). The authors should reflect on how these assumptions might be violated in practice and what the implications would be.
        \item The authors should reflect on the scope of the claims made, e.g., if the approach was only tested on a few datasets or with a few runs. In general, empirical results often depend on implicit assumptions, which should be articulated.
        \item The authors should reflect on the factors that influence the performance of the approach. For example, a facial recognition algorithm may perform poorly when image resolution is low or images are taken in low lighting. Or a speech-to-text system might not be used reliably to provide closed captions for online lectures because it fails to handle technical jargon.
        \item The authors should discuss the computational efficiency of the proposed algorithms and how they scale with dataset size.
        \item If applicable, the authors should discuss possible limitations of their approach to address problems of privacy and fairness.
        \item While the authors might fear that complete honesty about limitations might be used by reviewers as grounds for rejection, a worse outcome might be that reviewers discover limitations that aren't acknowledged in the paper. The authors should use their best judgment and recognize that individual actions in favor of transparency play an important role in developing norms that preserve the integrity of the community. Reviewers will be specifically instructed to not penalize honesty concerning limitations.
    \end{itemize}

\item {\bf Theory assumptions and proofs}
    \item[] Question: For each theoretical result, does the paper provide the full set of assumptions and a complete (and correct) proof?
    \item[] Answer: \answerYes{} 
    \item[] Justification: Yes, complete proofs are provided to ensure correctness and reproducibility.
    \item[] Guidelines: 
    \begin{itemize}
        \item The answer \answerNA{} means that the paper does not include theoretical results. 
        \item All the theorems, formulas, and proofs in the paper should be numbered and cross-referenced.
        \item All assumptions should be clearly stated or referenced in the statement of any theorems.
        \item The proofs can either appear in the main paper or the supplemental material, but if they appear in the supplemental material, the authors are encouraged to provide a short proof sketch to provide intuition. 
        \item Inversely, any informal proof provided in the core of the paper should be complemented by formal proofs provided in appendix or supplemental material.
        \item Theorems and Lemmas that the proof relies upon should be properly referenced. 
    \end{itemize}

    \item {\bf Experimental result reproducibility}
    \item[] Question: Does the paper fully disclose all the information needed to reproduce the main experimental results of the paper to the extent that it affects the main claims and/or conclusions of the paper (regardless of whether the code and data are provided or not)?
    \item[] Answer: \answerYes{} 
    \item[] Justification: Full implementation details, including experiment setup and hyper-parameters, are provided in Appendix \ref{sec:training_setup} respectively. Code is available at our repository: \url{https://anonymous.4open.science/r/fragile-ke/}.
    \item[] Guidelines:
    \begin{itemize}
        \item The answer \answerNA{} means that the paper does not include experiments.
        \item If the paper includes experiments, a \answerNo{} answer to this question will not be perceived well by the reviewers: Making the paper reproducible is important, regardless of whether the code and data are provided or not.
        \item If the contribution is a dataset and\slash or model, the authors should describe the steps taken to make their results reproducible or verifiable. 
        \item Depending on the contribution, reproducibility can be accomplished in various ways. For example, if the contribution is a novel architecture, describing the architecture fully might suffice, or if the contribution is a specific model and empirical evaluation, it may be necessary to either make it possible for others to replicate the model with the same dataset, or provide access to the model. In general. releasing code and data is often one good way to accomplish this, but reproducibility can also be provided via detailed instructions for how to replicate the results, access to a hosted model (e.g., in the case of a large language model), releasing of a model checkpoint, or other means that are appropriate to the research performed.
        \item While NeurIPS does not require releasing code, the conference does require all submissions to provide some reasonable avenue for reproducibility, which may depend on the nature of the contribution. For example
        \begin{enumerate}
            \item If the contribution is primarily a new algorithm, the paper should make it clear how to reproduce that algorithm.
            \item If the contribution is primarily a new model architecture, the paper should describe the architecture clearly and fully.
            \item If the contribution is a new model (e.g., a large language model), then there should either be a way to access this model for reproducing the results or a way to reproduce the model (e.g., with an open-source dataset or instructions for how to construct the dataset).
            \item We recognize that reproducibility may be tricky in some cases, in which case authors are welcome to describe the particular way they provide for reproducibility. In the case of closed-source models, it may be that access to the model is limited in some way (e.g., to registered users), but it should be possible for other researchers to have some path to reproducing or verifying the results.
        \end{enumerate}
    \end{itemize}

\item {\bf Open access to data and code}
    \item[] Question: Does the paper provide open access to the data and code, with sufficient instructions to faithfully reproduce the main experimental results, as described in supplemental material?
    \item[] Answer: \answerYes{} 
    \item[] Justification: We provide open access to our code (\url{https://anonymous.4open.science/r/fragile-ke/}), along with detailed reproduction instructions in Appendices \ref{sec:training_setup} to reproduce our experimental results. 
    \item[] Guidelines:
    \begin{itemize}
        \item The answer \answerNA{} means that paper does not include experiments requiring code.
        \item Please see the NeurIPS code and data submission guidelines (\url{https://neurips.cc/public/guides/CodeSubmissionPolicy}) for more details.
        \item While we encourage the release of code and data, we understand that this might not be possible, so \answerNo{} is an acceptable answer. Papers cannot be rejected simply for not including code, unless this is central to the contribution (e.g., for a new open-source benchmark).
        \item The instructions should contain the exact command and environment needed to run to reproduce the results. See the NeurIPS code and data submission guidelines (\url{https://neurips.cc/public/guides/CodeSubmissionPolicy}) for more details.
        \item The authors should provide instructions on data access and preparation, including how to access the raw data, preprocessed data, intermediate data, and generated data, etc.
        \item The authors should provide scripts to reproduce all experimental results for the new proposed method and baselines. If only a subset of experiments are reproducible, they should state which ones are omitted from the script and why.
        \item At submission time, to preserve anonymity, the authors should release anonymized versions (if applicable).
        \item Providing as much information as possible in supplemental material (appended to the paper) is recommended, but including URLs to data and code is permitted.
    \end{itemize}

\item {\bf Experimental setting/details}
    \item[] Question: Does the paper specify all the training and test details (e.g., data splits, hyperparameters, how they were chosen, type of optimizer) necessary to understand the results?
    \item[] Answer: \answerYes{} 
    \item[] Justification: We specify all hyperparameters, baselines \& other experimental details in Appendix \ref{sec:training_setup} and Section \ref{sec:results}.
    \item[] Guidelines:
    \begin{itemize}
        \item The answer \answerNA{} means that the paper does not include experiments.
        \item The experimental setting should be presented in the core of the paper to a level of detail that is necessary to appreciate the results and make sense of them.
        \item The full details can be provided either with the code, in appendix, or as supplemental material.
    \end{itemize}

\item {\bf Experiment statistical significance}
    \item[] Question: Does the paper report error bars suitably and correctly defined or other appropriate information about the statistical significance of the experiments?
    \item[] Answer: \answerNo{} 
    \item[] Justification: While we report results averaged over 5 independent runs, we do not include error bars or formal measures of statistical significance.
    \item[] Guidelines:
    \begin{itemize}
        \item The answer \answerNA{} means that the paper does not include experiments.
        \item The authors should answer \answerYes{} if the results are accompanied by error bars, confidence intervals, or statistical significance tests, at least for the experiments that support the main claims of the paper.
        \item The factors of variability that the error bars are capturing should be clearly stated (for example, train/test split, initialization, random drawing of some parameter, or overall run with given experimental conditions).
        \item The method for calculating the error bars should be explained (closed form formula, call to a library function, bootstrap, etc.)
        \item The assumptions made should be given (e.g., Normally distributed errors).
        \item It should be clear whether the error bar is the standard deviation or the standard error of the mean.
        \item It is OK to report 1-sigma error bars, but one should state it. The authors should preferably report a 2-sigma error bar than state that they have a 96\% CI, if the hypothesis of Normality of errors is not verified.
        \item For asymmetric distributions, the authors should be careful not to show in tables or figures symmetric error bars that would yield results that are out of range (e.g., negative error rates).
        \item If error bars are reported in tables or plots, the authors should explain in the text how they were calculated and reference the corresponding figures or tables in the text.
    \end{itemize}

\item {\bf Experiments compute resources}
    \item[] Question: For each experiment, does the paper provide sufficient information on the computer resources (type of compute workers, memory, time of execution) needed to reproduce the experiments?
    \item[] Answer: \answerYes{} 
    \item[] Justification: We provide comprehensive details of the compute settings for our method in Appendix \ref{sec:training_setup}.
    \item[] Guidelines:
    \begin{itemize}
        \item The answer \answerNA{} means that the paper does not include experiments.
        \item The paper should indicate the type of compute workers CPU or GPU, internal cluster, or cloud provider, including relevant memory and storage.
        \item The paper should provide the amount of compute required for each of the individual experimental runs as well as estimate the total compute. 
        \item The paper should disclose whether the full research project required more compute than the experiments reported in the paper (e.g., preliminary or failed experiments that didn't make it into the paper). 
    \end{itemize}
    
\item {\bf Code of ethics}
    \item[] Question: Does the research conducted in the paper conform, in every respect, with the NeurIPS Code of Ethics \url{https://neurips.cc/public/EthicsGuidelines}?
    \item[] Answer: \answerYes{} 
    \item[] Justification: Our paper aligns with the NeurIPS Code of Ethics.
    \item[] Guidelines:
    \begin{itemize}
        \item The answer \answerNA{} means that the authors have not reviewed the NeurIPS Code of Ethics.
        \item If the authors answer \answerNo, they should explain the special circumstances that require a deviation from the Code of Ethics.
        \item The authors should make sure to preserve anonymity (e.g., if there is a special consideration due to laws or regulations in their jurisdiction).
    \end{itemize}

\item {\bf Broader impacts}
    \item[] Question: Does the paper discuss both potential positive societal impacts and negative societal impacts of the work performed?
    \item[] Answer: \answerYes{} 
    \item[] Justification: This paper addresses both the potential positive and negative societal impacts in the Appendix \ref{sec:limits}.
    \item[] Guidelines:
    \begin{itemize}
        \item The answer \answerNA{} means that there is no societal impact of the work performed.
        \item If the authors answer \answerNA{} or \answerNo, they should explain why their work has no societal impact or why the paper does not address societal impact.
        \item Examples of negative societal impacts include potential malicious or unintended uses (e.g., disinformation, generating fake profiles, surveillance), fairness considerations (e.g., deployment of technologies that could make decisions that unfairly impact specific groups), privacy considerations, and security considerations.
        \item The conference expects that many papers will be foundational research and not tied to particular applications, let alone deployments. However, if there is a direct path to any negative applications, the authors should point it out. For example, it is legitimate to point out that an improvement in the quality of generative models could be used to generate Deepfakes for disinformation. On the other hand, it is not needed to point out that a generic algorithm for optimizing neural networks could enable people to train models that generate Deepfakes faster.
        \item The authors should consider possible harms that could arise when the technology is being used as intended and functioning correctly, harms that could arise when the technology is being used as intended but gives incorrect results, and harms following from (intentional or unintentional) misuse of the technology.
        \item If there are negative societal impacts, the authors could also discuss possible mitigation strategies (e.g., gated release of models, providing defenses in addition to attacks, mechanisms for monitoring misuse, mechanisms to monitor how a system learns from feedback over time, improving the efficiency and accessibility of ML).
    \end{itemize}
    
\item {\bf Safeguards}
    \item[] Question: Does the paper describe safeguards that have been put in place for responsible release of data or models that have a high risk for misuse (e.g., pre-trained language models, image generators, or scraped datasets)?
    \item[] Answer: \answerNA{} 
    \item[] Justification: N/A. The paper operates on existing models and synthetic or previously published data, and only code for experimentation is provided.
    \item[] Guidelines:
    \begin{itemize}
        \item The answer \answerNA{} means that the paper poses no such risks.
        \item Released models that have a high risk for misuse or dual-use should be released with necessary safeguards to allow for controlled use of the model, for example by requiring that users adhere to usage guidelines or restrictions to access the model or implementing safety filters. 
        \item Datasets that have been scraped from the Internet could pose safety risks. The authors should describe how they avoided releasing unsafe images.
        \item We recognize that providing effective safeguards is challenging, and many papers do not require this, but we encourage authors to take this into account and make a best faith effort.
    \end{itemize}

\item {\bf Licenses for existing assets}
    \item[] Question: Are the creators or original owners of assets (e.g., code, data, models), used in the paper, properly credited and are the license and terms of use explicitly mentioned and properly respected?
    \item[] Answer: \answerYes{} 
    \item[] Justification: We cite all baselines, datasets and data sources in our paper.
    \item[] Guidelines:
    \begin{itemize}
        \item The answer \answerNA{} means that the paper does not use existing assets.
        \item The authors should cite the original paper that produced the code package or dataset.
        \item The authors should state which version of the asset is used and, if possible, include a URL.
        \item The name of the license (e.g., CC-BY 4.0) should be included for each asset.
        \item For scraped data from a particular source (e.g., website), the copyright and terms of service of that source should be provided.
        \item If assets are released, the license, copyright information, and terms of use in the package should be provided. For popular datasets, \url{paperswithcode.com/datasets} has curated licenses for some datasets. Their licensing guide can help determine the license of a dataset.
        \item For existing datasets that are re-packaged, both the original license and the license of the derived asset (if it has changed) should be provided.
        \item If this information is not available online, the authors are encouraged to reach out to the asset's creators.
    \end{itemize}

\item {\bf New assets}
    \item[] Question: Are new assets introduced in the paper well documented and is the documentation provided alongside the assets?
    \item[] Answer: \answerYes{} 
    \item[] Justification: We present detailed documentation about our code in our anonymous repository: \url{https://anonymous.4open.science/r/fragile-ke/}.
    \item[] Guidelines:
    \begin{itemize}
        \item The answer \answerNA{} means that the paper does not release new assets.
        \item Researchers should communicate the details of the dataset\slash code\slash model as part of their submissions via structured templates. This includes details about training, license, limitations, etc. 
        \item The paper should discuss whether and how consent was obtained from people whose asset is used.
        \item At submission time, remember to anonymize your assets (if applicable). You can either create an anonymized URL or include an anonymized zip file.
    \end{itemize}

\item {\bf Crowdsourcing and research with human subjects}
    \item[] Question: For crowdsourcing experiments and research with human subjects, does the paper include the full text of instructions given to participants and screenshots, if applicable, as well as details about compensation (if any)? 
    \item[] Answer: \answerNA{} 
    \item[] Justification: Not applicable for this work.
    \item[] Guidelines:
    \begin{itemize}
        \item The answer \answerNA{} means that the paper does not involve crowdsourcing nor research with human subjects.
        \item Including this information in the supplemental material is fine, but if the main contribution of the paper involves human subjects, then as much detail as possible should be included in the main paper. 
        \item According to the NeurIPS Code of Ethics, workers involved in data collection, curation, or other labor should be paid at least the minimum wage in the country of the data collector. 
    \end{itemize}

\item {\bf Institutional review board (IRB) approvals or equivalent for research with human subjects}
    \item[] Question: Does the paper describe potential risks incurred by study participants, whether such risks were disclosed to the subjects, and whether Institutional Review Board (IRB) approvals (or an equivalent approval/review based on the requirements of your country or institution) were obtained?
    \item[] Answer: \answerNA{} 
    \item[] Justification: Not applicable for this work.
    \item[] Guidelines:
    \begin{itemize}
        \item The answer \answerNA{} means that the paper does not involve crowdsourcing nor research with human subjects.
        \item Depending on the country in which research is conducted, IRB approval (or equivalent) may be required for any human subjects research. If you obtained IRB approval, you should clearly state this in the paper. 
        \item We recognize that the procedures for this may vary significantly between institutions and locations, and we expect authors to adhere to the NeurIPS Code of Ethics and the guidelines for their institution. 
        \item For initial submissions, do not include any information that would break anonymity (if applicable), such as the institution conducting the review.
    \end{itemize}

\item {\bf Declaration of LLM usage}
    \item[] Question: Does the paper describe the usage of LLMs if it is an important, original, or non-standard component of the core methods in this research? Note that if the LLM is used only for writing, editing, or formatting purposes and does \emph{not} impact the core methodology, scientific rigor, or originality of the research, declaration is not required.
    \item[] Answer: \answerYes{} 
    \item[] Justification: This paper clearly describes the usage of LLMs as a key component of our work, especially since we discuss the fragilities of KE in LLMs.
    \item[] Guidelines:
    \begin{itemize}
        \item The answer \answerNA{} means that the core method development in this research does not involve LLMs as any important, original, or non-standard components.
        \item Please refer to our LLM policy in the NeurIPS handbook for what should or should not be described.
    \end{itemize}

\end{enumerate}
\end{document}